\def\year{2018}\relax
\documentclass[letterpaper]{article}
\usepackage{aaai18}
\usepackage{times}
\usepackage{helvet}
\usepackage{courier}

\usepackage{epsfig}
\usepackage{amsthm}
\usepackage{amssymb}
\usepackage{amsmath}
\usepackage{amsfonts}
\usepackage{ifthen}
	
\usepackage{graphicx}
\usepackage{caption}
\usepackage{subcaption}
\usepackage{epstopdf}
\usepackage{wrapfig}
\usepackage{bm}
\usepackage{booktabs}       
\usepackage{nicefrac}       
\usepackage{array}			
\usepackage{url}

\newcommand{\argmax}{ \operatorname*{arg\,max}}

\newtheorem{theorem}{Theorem}

\newtheorem{lemma}{Lemma}

\newtheorem{definition}{Definition}
\newtheorem{assumption}{Assumption}

\newcommand{\set}[1]{{\mathcal #1}}

\def\myproof{1} 
\if\myproof1\nocopyright\fi

\begin{document}
%
\title{Decentralized High-Dimensional Bayesian Optimization with Factor Graphs}
\author{
Trong Nghia Hoang$^{\S}$ \and Quang Minh Hoang$^{\dag}$ \and Ruofei Ouyang$^{\dag}$ \and Kian Hsiang Low$^{\dag}$\\
Laboratory of Information and Decision Systems, Massachusetts Institute of Technology, USA$^{\S}$\\ 
Department of Computer Science, National University of Singapore, Republic of Singapore$^{\dag}$\\
nghiaht@mit.edu$^{\S}$, \{hqminh, ouyang, lowkh\}@comp.nus.edu.sg$^{\dag}$ 
}

\maketitle

\begin{abstract}
This paper presents a novel \emph{decentralized high-dimensional Bayesian optimization} (DEC-HBO) algorithm that, in contrast to existing HBO algorithms, can exploit the interdependent effects of various input components on the output of the unknown objective function $f$ for boosting the BO performance and still preserve scalability in the number of input dimensions without requiring prior knowledge or the existence of a low (effective) dimension of the input space.
To realize this, we propose a sparse yet rich factor graph representation of $f$ to be exploited for designing an acquisition function that can be similarly represented by a sparse factor graph and hence be efficiently optimized in a decentralized manner using distributed message passing.
Despite richly characterizing the interdependent effects of the input components on the output of $f$ with a factor graph, DEC-HBO can still guarantee no-regret performance asymptotically.
Empirical evaluation on synthetic and real-world experiments (e.g., sparse Gaussian process model with $1811$ hyperparameters) shows that DEC-HBO outperforms the state-of-the-art HBO algorithms.
\end{abstract}

\section{Introduction}
\label{intro}
Many real-world applications/tasks often involve optimizing an unknown objective function $f$ given a limited budget of costly function evaluations. Examples of such applications/tasks include 
automatic hyperparameter tuning for machine learning models (e.g., deep neural network) \cite{Yamins13,Snoek12} and parameter configuration for robotic control strategies. 
Whereas gradient-based methods fail to optimize a function without an analytic form/derivative,
 \emph{Bayesian optimization} (BO) has established itself as a highly effective alternative. 
In particular, a BO algorithm maintains a \emph{Gaussian process} (GP) belief of the unknown objective function $f$ and alternates between selecting an input query to evaluate $f$ and using 
its observed output 
to update the GP belief of $f$ until the budget is exhausted. Every input query is selected by maximizing an acquisition function that is constructed from the GP belief of $f$.
Intuitively, such an acquisition function has to trade off between optimizing $f$ based on its current GP belief (exploitation) vs. improving its GP belief (exploration). Popular choices include improvement-based~\cite{shahriari16},  
information-based \cite{Hennig12},
and upper confidence bound~\cite{Srinivas10}.

While BO has demonstrated to be an effective optimization strategy in general, it has mostly found success in the context of input spaces with low (effective) dimension	 \cite{Djolonga13,Wang13}. 
However, several real-world application domains such as computer vision \cite{Yamins13}, networking \cite{Hornby06}, and computational biology \cite{Gonzalez14} often require optimizing an objective function $f$ over a high-dimensional input space without 
knowing if  its low (effective) dimension even exists.
This poses a grand challenge to the above conventional BO algorithms as the cost of maximizing an acquisition function grows exponentially with the number of input dimensions. To sidestep this issue, an extreme approach is to assume the effects of all input components on the output of $f$ to be pairwise independent~\cite{Kandasamy15,Wang17b,Zi17} 
(in the case of~\citeauthor{Li16}~\shortcite{Li16}, after some affine projection of the input space).
The acquisition function can then be maximized along each (projected) dimension separately, thus reducing its cost to linear in the number of input dimensions. Despite its simplicity, such a decoupling assumption can severely compromise the BO performance since it rarely holds in practice: 
The effects of different input components on the output of $f$ are usually interdependent~\cite{Klemen12}.
%
%
%
In this paper, we argue and show that this highly restrictive assumption to gain scalability is an overkill: It is in fact possible to achieve the same scalability with a strong performance guarantee while still taking into account the interdependent effects of various input components on the output of $f$. 

To achieve this, we first observe that the interdependent effects of many input components on the output of $f$ tend to be indirect: The effect of one input component (on the output of $f$) can only directly influence that of some components in its immediate ``neighborhood'', which in turn may influence that of other components in the same manner. For example, in a multi-project company, the poor performance of one employee in a collaborative project only indirectly affects the performance of another employee in another project through those who work on both projects. This is also the case for many parameter tuning tasks with additive loss where different, overlapping subsets of parameters contribute to different additive factors of the loss function \cite{Krahenbuhl11}.
The key challenge thus lies in investigating how the unknown objective function $f$ can be succinctly modeled to characterize such observed interdependent effects of the input components (on the output of $f$) and then exploited to design an acquisition function that can still be optimized scalably and yield a provable performance guarantee.

To address this challenge, this paper presents a novel \emph{decentralized high-dimensional BO} (DEC-HBO) algorithm (Section~\ref{problem}) that, in contrast to some HBO algorithms~\cite{Kandasamy15,Wang17b,Zi17}, can exploit the interdependent effects of various input components (on the output of $f$) for boosting the BO performance and, perhaps surprisingly, still preserve scalability in the number of input dimensions without requiring 
the existence of a low (effective) dimension of the input space, unlike the other HBO algorithms~\cite{Djolonga13,Li16,Wang13}.
To realize this, we propose a sparse yet rich and highly expressive factor graph representation of the unknown objective function $f$ by decomposing it into a linear sum of random factor functions, each of which involves only a small, \emph{possibly overlapping} subset of input components and is assumed to be distributed by an independent GP prior (Section~\ref{anova}). 
As a result, the input components of the same factor function have direct interdependent effects on the output of $f$, while the input components that are distinct between any two factor functions have indirect interdependent effects on the output of $f$ via their common input components;
the latter is predominant due to sparsity of the factor graph representing $f$.
%
We in turn exploit such a factor graph representation of $f$ to design an acquisition function that, interestingly, can be similarly represented by a sparse factor graph (Section~\ref{acquisition}) and hence be efficiently optimized (Section~\ref{dBO}) in a decentralized manner using a class of distributed message passing algorithms.
The main novel contribution of our work here is to show that despite richly characterizing the interdependent effects of the input components on the output of $f$ with a factor graph,
our DEC-HBO algorithm not only preserves the scalability in the number of input dimensions
but also guarantees the same trademark (asymptotic) no-regret performance. 
We empirically demonstrate the performance of DEC-HBO with synthetic and real-world experiments (e.g., sparse Gaussian process model with $1811$ hyperparameters) (Section~\ref{exp}).
\section{Background and Notations}
\label{background}
This section first describes the zeroth-order optimization problem and its asymptotic optimality criterion which lay the groundwork for BO. Then, we review a class of well-studied BO algorithms~\cite{Kandasamy15,Srinivas10} and highlight their practical limitations when applied to high-dimensional optimization problems. 
We will discuss later in Sections~\ref{problem} and~\ref{analysis} how our proposed DEC-HBO algorithm overcomes these limitations.
%
\subsection{Zeroth-Order Optimization}
\label{bo}
Consider the problem of sequentially optimizing an unknown objective function $f: \set{D} \rightarrow \mathbb{R}$ over a compact input domain $\mathcal{D} \subseteq \mathbb{R}^d$:
In each iteration $t=1,\ldots,n$, an input query $\mathbf{x}_t\in\mathcal{D}$ is selected for evaluating $f$ to yield a noisy observed output $y_t \triangleq f(\mathbf{x}_t) + \epsilon$ with i.i.d. Gaussian noise $\epsilon\sim\mathcal{N}(0, \sigma_n^2)$ and noise variance $\sigma_n^2$. 
Since every evaluation of $f$ is costly (Section~\ref{intro}), our goal is to strategically select input queries to approach the global maximizer ${\mathbf{x}}_\ast \triangleq \arg\max_{{\mathbf{x}} \in \set{D}} f({\mathbf{x}})$ as rapidly as possible.
This can be achieved by minimizing a standard BO objective such as the \emph{cumulative regret} $R_n$
which sums the \emph{instantaneous regret} $r_t \triangleq f(\mathbf{x}_\ast) - f(\mathbf{x}_t)$ incurred by selecting the input query $\mathbf{x}_t$ (instead of $\mathbf{x}_\ast$ due to not knowing $\mathbf{x}_\ast$ beforehand) to evaluate $f$ over iteration $t=1,\ldots,n$,
that is, $R_n \triangleq \sum^n_{t=1} r_t$.
A BO algorithm is said to be \emph{asymptotically optimal} if it satisfies $\mathrm{lim}_{n\rightarrow\infty} R_n/n = 0$ which implies $\mathrm{lim}_{n\rightarrow\infty}(f(\mathbf{x}_\ast) - \max_{t=1}^n f(\mathbf{x}_t)) = 0$, thus guaranteeing \emph{no-regret} performance asymptotically.
\subsection{Bayesian Optimization with No Regret}
\label{gpucb}
A notable asymptotically optimal BO algorithm selects, in each iteration $t+1$, an input query $\mathbf{x} \in \mathcal{D}$ 
to maximize an acquisition function called the \emph{Gaussian process upper confidence bound} (GP-UCB) \cite{Srinivas10} 
that trades off between observing an expected maximum (i.e., with large GP posterior mean $\mu_t(\mathbf{x})$) given the current GP belief of $f$ (i.e., exploitation) vs. that of high predictive uncertainty (i.e., with large GP posterior variance $\sigma_t(\mathbf{x})^2$) to
improve the GP belief of $f$ over $\set{D}$ (i.e., exploration), that is, 
$
\mathbf{x}_{t+1} \triangleq \argmax_{\mathbf{x} \in \mathcal{D}} \mu_t(\mathbf{x}) + {\beta^{1/2}_{t+1}}\sigma_t(\mathbf{x})
$
where the parameter $\beta_{t+1} > 0$ is set to trade off between exploitation vs. exploration for 
guaranteeing no-regret performance asymptotically with high probability
and the GP posterior mean $\mu_t(\mathbf{x})$ and variance $\sigma_t(\mathbf{x})^2$ will be defined later in a similar manner to~\eqref{eq:9} (Section~\ref{acquisition}) to ease exposition.

Unfortunately, the GP-UCB algorithm does not scale well to high-dimensional optimization problems as its cost grows exponentially with the number of input dimensions. 
This prohibits its use in real-world application domains that require optimizing an objective function over a high-dimensional input space such as those mentioned in Section~\ref{intro}.
To sidestep this issue, some HBO algorithms~\cite{Djolonga13,Li16,Wang13} assume the existence of a low-dimensional embedding of the input space which then allows them to operate in an exponentially smaller surrogate space and hence reduce their cost. 
But, these HBO algorithms impose strong assumptions (including prior knowledge of the   dimension of the embedding) to guarantee that the global maximizer (or its affine projection) indeed lies within the surrogate space. In particular, one such precarious assumption is that the dimensionality of the low-rank surrogate space reflects the actual effective dimension of the input space.

A more practical alternative is to consider the effects of various input components on the output of $f$ instead:
The HBO algorithm of~\citeauthor{Kandasamy15}~\shortcite{Kandasamy15} assumes the unknown objective function $f$ to be decomposable into a sum of independent, GP-distributed local functions $f_1,\ldots, f_d$, each of which involves only a single input dimension: $f(\mathbf{x}) \triangleq f_1(\mathbf{x}^{(1)}) + \ldots + f_d(\mathbf{x}^{(d)})$ where $\mathbf{x}^{(i)}$ and $d$ denote component $i$ and the dimension of input $\mathbf{x}$, respectively. Interestingly, this in turn induces a similar decomposition of the above-mentioned GP-UCB acquisition function into a sum 
of independent local acquisition functions $\varphi^{(i)}_{t}(\mathbf{x}^{(i)})\triangleq \mu^{(i)}_t(\mathbf{x}^{(i)}) + {\beta^{1/2}_{t+1}}\sigma^{(i)}_t(\mathbf{x}^{(i)})$ for $i=1,\ldots,d$, that is, $\sum_{i=1}^d \varphi^{(i)}_{t}(\mathbf{x}^{(i)})$.
As a result, each local acquisition function $\varphi^{(i)}_{t}(\mathbf{x}^{(i)})$ can be independently maximized along a separate input dimension,
thus reducing the overall computational cost to linear
in the number $d$ of input dimensions. 
However, such a HBO algorithm and a few others~\cite{Wang17b,Zi17} preclude the interdependent effects of different input dimensions on the output of $f$~\cite{Klemen12}, which can severely compromise their performance.
We will describe in Sections~\ref{problem} and~\ref{analysis} how our DEC-HBO algorithm can exploit the interdependent effects of various input components (on the output of $f$) for boosting the BO performance and still preserve scalability in the number of input dimensions as well as guarantee no-regret performance asymptotically.
%
%
%
\section{Problem Formulation}
\label{problem}
This section first introduces the sparse factor representation of the unknown objective function $f$ (Section~\ref{anova}). Then, we exploit it to reformulate the GP-UCB acquisition function to a form that can be similarly represented by a sparse factor graph (Section~\ref{acquisition})
and hence be efficiently optimized in a decentralized manner using distributed message passing (Section~\ref{dBO}) to achieve the same trademark (asymptotic) no-regret performance guarantee (Section~\ref{analysis}).
%
\subsection{Sparse Factor Graph Representation}
\label{anova}
To scale up BO to high-dimensional optimization problems 
while still taking into account the interdependent effects of various input components on the output of $f$, we first state the following key structural assumption to represent $f$ as a sparse yet rich and highly expressive factor graph:

\begin{assumption} 
The $d$-dimensional objective function $f$ can be decomposed into 
a sum of $|\mathcal{U}|$ factor functions 
$\{f_{\mathcal{I}}\}_{\mathcal{I}\in\mathcal{U}}$, each of which depends on a $|\mathcal{I}|$-dimensional input $\mathbf{x}^{\mathcal{I}}$ comprising only a small, \emph{possibly overlapping} subset $\mathcal{I}\subseteq\mathcal{S} \triangleq \{1, 2, \ldots, d\}$ of the input components of $\mathbf{x}$ (i.e., $|\mathcal{I}| \ll d$), that is, $f(\mathbf{x})\triangleq\sum_{\mathcal{I}\in\mathcal{U}} f_{\mathcal{I}}(\mathbf{x}^{\mathcal{I}})$.
\end{assumption}

\noindent
Intuitively, Assumption 1 decomposes the high $d$-dimensional optimization problem into small sub-problems, each of which involves optimizing a single factor function $f_{\mathcal{I}}$ over a low $|\mathcal{I}|$-dimensional input space (and is hence much less costly) while succinctly encoding the compatibility of its selected input $\mathbf{x}^{\mathcal{I}}$ with that of other factor functions through their common input components;
the latter is completely disregarded by the state-of-the-art HBO algorithms~\cite{Kandasamy15,Wang17b,Zi17} due to their highly restrictive decoupling assumption (Section~\ref{gpucb}). 
In practice, our more relaxed assumption thus allows prior knowledge of the interdependent effects of different input components (on the output of $f$) to be explicitly and succinctly encoded into the sparse factor graph representation of $f$.
As a result, the input components of the same factor function have direct interdependent effects on the output of $f$, while the input components that are distinct between any two factor functions have indirect interdependent effects on the output of $f$ via their common input components;
the latter is predominant due to sparsity of the factor graph representing $f$.
Interestingly, our assumption can be further coupled with Assumption 2 below to induce an additive GP model of $f$~\cite{Duvenaud11} with truncated ANOVA kernels which have shown to be highly expressive in characterizing the latent interaction between different input components.
%

Despite needing to maintain the compatibility of their selected inputs, 
these factor functions can still be optimized in a decentralized manner if a message passing protocol can be established between them to allow those with common input components to coordinate their optimization efforts
without requiring any factor function to handle input components not of its own.
To achieve this, two non-trivial research questions arise: Firstly, how can these factor functions (with compatibility constraints) without 
analytic expressions nor 
black-box generators be optimized (see Sections~\ref{acquisition} and~\ref{dBO})? 
Secondly, even if it is possible to optimize each factor function, how can their coordinated optimization efforts be guaranteed to converge the selected input queries to the global maximizer of $f$ (see Section~\ref{analysis})?  
Note that the second question has not been addressed by the previously established convergence guarantee for the GP-UCB algorithm~\cite{Srinivas10} as it only applies to the centralized setting but not our decentralized BO setting.
%
%
\subsection{Acquisition Function}
\label{acquisition}
To optimize each factor function without having direct access to its black-box generator, we need a mechanism that can draw inference on the output of the factor function $f_{\mathcal{I}}$ given only the noisy observed outputs of $f$. This is achieved with the following assumption:
\begin{assumption} 
Each factor function $f_{\mathcal{I}}$ in the decomposition of $f$ 
in Assumption 1 is independently distributed by a GP $\mathcal{GP}(0, \sigma^{\mathcal{I}}_0(\mathbf{x}^{\mathcal{I}}, \mathbf{x}'^{\mathcal{I}}))$ with prior mean $\mu^{\mathcal{I}}_0(\mathbf{x}^{\mathcal{I}}) \triangleq 0$ and covariance $\sigma^{\mathcal{I}}_0(\mathbf{x}^{\mathcal{I}}, \mathbf{x}'^{\mathcal{I}})$.
\end{assumption}

\noindent
Assumption 2 implies that $f$ is distributed by a GP $\mathcal{GP}(0, \sigma_0(\mathbf{x}, \mathbf{x}'))$ with prior mean $0$ and covariance $\sigma_0(\mathbf{x}, \mathbf{x}')\triangleq \sum_{\mathcal{I}\in\mathcal{U}}\sigma^{\mathcal{I}}_0(\mathbf{x}^{\mathcal{I}}, \mathbf{x}'^{\mathcal{I}})$. 
It follows that for any subset $\mathcal{I} \subseteq \mathcal{S}$ of the input components of any input $\mathbf{x}$  and input queries $\mathbf{x}_1,\ldots,\mathbf{x}_t$,
the prior distribution of $(f_\mathcal{I}(\mathbf{x}^{\mathcal{I}}), f(\mathbf{x}_1),\ldots, f(\mathbf{x}_t))^\top$ is a Gaussian. 
Then, given a column vector $\mathbf{y} \triangleq (y_i)^{\top}_{i=1,\ldots,t})$ of noisy outputs observed from evaluating $f$ at the selected input queries $\mathbf{x}_1,\ldots,\mathbf{x}_t$ after $t$ iterations, the posterior distribution of the output of the factor function $f_{\mathcal{I}}$ at some input $\mathbf{x}^{\mathcal{I}}$ in iteration $t+1$ is a Gaussian $\mathcal{N}(f_{\mathcal{I}}(\mathbf{x}^{\mathcal{I}})|\mu_{t}^\mathcal{I}(\mathbf{x}^{\mathcal{I}}), \sigma_{t}^\mathcal{I}(\mathbf{x}^{\mathcal{I}})^2)$ with the following posterior mean and variance: 
\begin{equation}
\begin{array}{rcl}
\mu_{t}^\mathcal{I}(\mathbf{x}^{\mathcal{I}}) &\triangleq & \mathbf{k}_{\mathbf{x}}^{\mathcal{I}^\top}(\mathbf{K} + \sigma_n^2\mathbf{I})^{-1}\mathbf{y} \ ,\\
\sigma_{t}^\mathcal{I}(\mathbf{x}^{\mathcal{I}})^2 &\triangleq & \sigma_0^\mathcal{I}(\mathbf{x}^{\mathcal{I}},\mathbf{x}^{\mathcal{I}}) - \mathbf{k}_{\mathbf{x}}^{\mathcal{I}^\top}(\mathbf{K} + \sigma_n^2\mathbf{I})^{-1}\mathbf{k}_{\mathbf{x}}^{\mathcal{I}}
\end{array}
\label{eq:9}
\end{equation}
where $\mathbf{k}_{\mathbf{x}}^{\mathcal{I}} \triangleq (\sigma_0^\mathcal{I}(\mathbf{x}^{\mathcal{I}}, \mathbf{x}_i^{\mathcal{I}}))^\top_{i=1,\ldots,t}$ and $\mathbf{K} \triangleq (\sigma_0(\mathbf{x}_i,\mathbf{x}_j))_{i,j=1,\ldots,t}$. 
Using~\eqref{eq:9}, 
we can naively adopt the 
HBO algorithm of~\citeauthor{Kandasamy15}~\shortcite{Kandasamy15} (Section~\ref{gpucb})
by independently maximizing a separate local acquisition function for every corresponding factor function $f_{\mathcal{I}}$.
But, this does not guarantee compatibility of the inputs selected by independently maximizing each local acquisition function due to their common input components. 
So, we instead propose to jointly maximize them using the following additive acquisition function:
\begin{equation}
\begin{array}{c}
\displaystyle\sum_{\mathcal{I}\in\mathcal{U}}\varphi^\mathcal{I}_{t}(\mathbf{x}^\mathcal{I})\ ,\quad\varphi^\mathcal{I}_{t}(\mathbf{x}^\mathcal{I}) \triangleq \mu_t^\mathcal{I}(\mathbf{x}^{\mathcal{I}}) + \beta_{t+1}^{1/2}\sigma_{t}^\mathcal{I}(\mathbf{x}^{\mathcal{I}}) 
\end{array}
\label{eq:11}
\end{equation}
which appears, with high probability, to bound the global maximum $f(\mathbf{x}_\ast) = \max_{\mathbf{x}\in\mathcal{X}}\sum_\mathcal{I}f_\mathcal{I}(\mathbf{x}^\mathcal{I})$ from above, as shown later in Section~\ref{analysis}. Intuitively,~\eqref{eq:11} is similar to the GP-UCB acquisition function~\cite{Srinivas10} in the sense that both are exploiting the GP posterior mean of $f$ to select the next input query since it can be shown that the sum of GP posterior means of the outputs of all factor functions is equal to the GP posterior mean of their sum (i.e., output of $f$). 
On the other hand, unlike GP-UCB, our proposed acquisition function~\eqref{eq:11} uses 
the sum of GP posterior variances of the outputs of all factor functions
instead of the GP posterior variance of their sum (i.e., output of $f$) in order to construct an upper bound on the global maximum. This interestingly allows~\eqref{eq:11} to be optimized efficiently in a decentralized manner (Section~\ref{dBO}) and at the same time preserves the asymptotic optimality of our DEC-HBO algorithm (Section~\ref{analysis}).
%
\subsection{Decentralized HBO (DEC-HBO)}
\label{dBO}
When the interdependent effects of the majority of input components on the output of $f$ are indirect, our proposed acquisition function~\eqref{eq:11} can in fact be represented by a sparse factor graph and hence
be efficiently optimized in a decentralized manner.
To do this, 
let~\eqref{eq:11} be represented by a factor graph with each factor and variable node denoting, respectively, a different local acquisition function and input component  
such that every edge connecting a factor node to some variable node
implies a local acquisition function involving the participation of some input component.
The following message passing protocol between the factor and variable nodes can then be used to optimize 
$\sum_{\mathcal{I}\in\mathcal{U}}\varphi_t^\mathcal{I}(\mathbf{x}^\mathcal{I})$~\eqref{eq:11} via \emph{dynamic programming} (DP):%
\subsubsection{Message Passing Protocol.} In iteration $t+1$, let $m_{\varphi_t^{\mathcal{I}} \rightarrow \mathbf{x}^{(i)}}(h)$ and $m_{\mathbf{x}^{(i)} \rightarrow \varphi_t^{\mathcal{I}}}(h)$ denote messages to be passed from a factor node $\varphi_t^{\mathcal{I}}(\mathbf{x}^\mathcal{I})$ (i.e., a local acquisition function) to a variable node $\mathbf{x}^{(i)}$ (i.e., component $i\in\mathcal{I}$ of its input $\mathbf{x}^\mathcal{I}$)
and from $\mathbf{x}^{(i)}$ back to $\varphi_t^{\mathcal{I}}(\mathbf{x}^\mathcal{I})$, respectively. Given $\mathbf{x}^{(i)} \triangleq h$,
\begin{equation}
\begin{array}{l}
m_{\varphi_t^{\mathcal{I}} \rightarrow \mathbf{x}^{(i)}}(h) 
\triangleq\displaystyle \max_{\mathbf{h}^{\mathcal{I}\setminus i}\in \mathcal{D}(\mathbf{x}^{\mathcal{I}\setminus i})} {\Delta}^{-i}_{\varphi_t^\mathcal{I}}(\mathbf{h}^{\mathcal{I}\setminus i}) + \varphi_t^\mathcal{I}(\mathbf{h}^{\mathcal{I}\setminus i}, h) ,\\
m_{\mathbf{x}^{(i)} \rightarrow \varphi_t^{\mathcal{I}}}(h) \triangleq \displaystyle\sum_{\mathcal{I}'\in\mathcal{A}(i)\setminus\{\mathcal{I}\}} m_{\varphi_t^{\mathcal{I}'} \rightarrow \mathbf{x}^{(i)}}(h)
\end{array}
\label{eq:12}
\end{equation}
where $\mathcal{I}\backslash i$ is used in place of $\mathcal{I}\backslash\{i\}$ to ease notations,
${\Delta}^{-i}_{\varphi_t^\mathcal{I}}(\mathbf{h}^{\mathcal{I}\setminus i}) \triangleq \sum_{j \in \mathcal{I}\setminus{i}} m_{\mathbf{x}^{(j)} \rightarrow \varphi_t^\mathcal{I}}(\mathbf{h}^{(j)})$, 
$\mathcal{D}(\mathbf{x}^{\mathcal{I}\setminus i})$ is the domain of input $\mathbf{x}^{\mathcal{I}\setminus i}$, 
$\mathbf{h}^{(j)}$ denotes component $j$ of $\mathbf{h}^{\mathcal{I}\setminus i}$, 
and $\mathcal{A}(i)\triangleq\{\mathcal{I}'|f_{\mathcal{I}'}(\mathbf{x}^{\mathcal{I}'})\ \text{is a factor function}\wedge i\in\mathcal{I}'\}$. 
%
%
These message updates~\eqref{eq:12} can be performed simultaneously, which yields a fully decentralized optimization algorithm where full knowledge of the results of optimization are accessible to all nodes. 
This can be perceived as a concurrent learning process where each node tries to perfect its own DP perspective through exchanging information with its immediate neighbors. As the messages are passed back and forth simultaneously among nodes, their individual perspectives are updated and steadily converge to an equilibrium that maximizes $\sum_{\mathcal{I}\in\mathcal{U}}\varphi_t^\mathcal{I}(\mathbf{x}^\mathcal{I})$~\eqref{eq:11}. 
Our decentralized optimization algorithm yields a huge computational advantage in a sparse factor graph since the cost of evaluating each message in any iteration is only as expensive as iterating through the input domain of the local acquisition function involving the largest number of input components (i.e., maximum factor size), 
which is usually much smaller than the entire input domain. 
The overall computational cost at each node thus reduces at an exponential rate in the ratio between the sizes of the original input domain and that of such a local acquisition function.

Upon convergence\footnote{Though the convergence of our message passing protocol is only guaranteed when the factor graph is a tree, it empirically converges to a competitive performance quickly (Section~\ref{hypertuning}). 
To guarantee the performance for a general factor graph, a bounded variant of the max-sum algorithm \cite{Rogers11} can be considered.}, the final message $m_{\varphi_t^{\mathcal{I}'} \rightarrow \mathbf{x}^{(i)}}(h)$ from every factor node $\varphi_t^{\mathcal{I}'}(\mathbf{x}^{\mathcal{I}'})$ to a variable node 
$\mathbf{x}^{(i)}$ 
(i.e., component $i\in\mathcal{I}'$ of its input $\mathbf{x}^{\mathcal{I}'}$)
is the maximum value achieved by optimizing the sum of all remaining factor nodes, except $\varphi_t^\mathcal{I}(\mathbf{x}^\mathcal{I})$, over the remaining variable nodes $\mathbf{x}^{\mathcal{S}\setminus i}$ while fixing $\mathbf{x}^{(i)} = h$. 
As such, component $i$ of the maximizer $\mathbf{x}_{t+1} \triangleq \argmax_{\mathbf{x}\in\mathcal{D}} \sum_{\mathcal{I}}\varphi_t^\mathcal{I}(\mathbf{x}^\mathcal{I})$ can be computed using an arbitrary variable-factor pair $(\mathbf{x}^{(i)}, \varphi_t^\mathcal{I}(\mathbf{x}^\mathcal{I}))$: 
\begin{equation}
\mathbf{x}_{t+1}^{(i)} \triangleq \argmax_{h \in \mathcal{D}(\mathbf{x}^{(i)})} \max_{\mathbf{h}^{\mathcal{I}\setminus i}\in \mathcal{D}(\mathbf{x}^{\mathcal{I}\setminus i})} \varphi_t^\mathcal{I}(\mathbf{h}^{\mathcal{I}\setminus i}, h) + m_{\mathbf{x}^{(i)} \rightarrow \varphi_t^\mathcal{I}}(h) 
\label{eq:13}
\end{equation}
for all $i \in \mathcal{I}$ where $\mathcal{D}(\mathbf{x}^{(i)})$ denotes the domain of input component $\mathbf{x}^{(i)}$. Note that~\eqref{eq:13} only operates in the domains $\mathcal{D}(\mathbf{x}^{(i)})$ and $\mathcal{D}(\mathbf{x}^{\mathcal{I}\setminus i})$.
Its time complexity is thus bounded by the cost of iterating through the input domain of the local acquisition function involving the largest number of input components (i.e., maximum factor size), as  analyzed in\if\myproof1 Appendix~\ref{app:e}. \fi\if\myproof0 \cite{Nghia17}. \fi 
Our decentralized optimization algorithm is similar in spirit to the max-sum algorithm for solving the well-known {distributed constraint optimization problem}~\cite{Leite14} operating in discrete input domains and in fact adapts it
to maximize our additive acquisition function~\eqref{eq:11}
over a continuous input domain.
Such an adaptation is achieved by scheduling an iterative process of domain discretization with increasing granularity. 
Nevertheless, our DEC-HBO algorithm is guaranteed to be asymptotically optimal, as further detailed in Section~\ref{continuous}. 
DEC-HBO requires a specification of the input partition $\mathcal{U} \subseteq 2^{\mathcal{S}}$ that underlies our additive acquisition function~\eqref{eq:11}, which can be learned from data\if\myproof1 (Appendix~\ref{fgr}). \fi\if\myproof0 \cite{Nghia17}. \fi 
%
%
\section{Asymptotic Optimality}
\label{analysis}
This section analyzes the asymptotic optimality of our proposed DEC-HBO algorithm (Section~\ref{dBO}) that is powered by our additive acquisition function~\eqref{eq:11}. We will first present a simplified version of our analysis in a simple setting with discrete input domains 
and then generalize it to handle a more realistic setting with continuous input domains. 
%
\subsubsection{Discrete Input Space.}
\label{discrete}
To guarantee the asymptotic optimality of DEC-HBO, it suffices to show that its average regret approaches zero in the limit (i.e., $\mathrm{lim}_{n\rightarrow\infty} R_n/n = 0$). To achieve this, we will first construct upper bounds for the instantaneous regrets $r_t = f(\mathbf{x}_\ast) - f(\mathbf{x}_t)$ (Theorem~\ref{theo:1}) and then combine these results to establish a sub-linear upper bound for the cumulative regret $R_n$ (Theorem~\ref{theo:2}).
\begin{theorem}
\label{theo:1}
Given $\delta \in (0, 1)$, let $\beta_t \triangleq 2\log(|\mathcal{D}||\mathcal{U}|\pi_t/\delta)$ with $\pi_t \triangleq \pi^2t^2 / 6$.\\  
$\mathrm{Pr}(\forall \mathbf{x}\in\mathcal{D}, t\in\mathbb{N} \ \ r_t \leq 2\beta_t^{1/2}\sum_{\mathcal{I}\in\mathcal{U}} \sigma_{t-1}^\mathcal{I}(\mathbf{x}_t^\mathcal{I}))\geq 1 - \delta$.
\end{theorem}
Its proof is in\if\myproof1 Appendix~\ref{app:a}. \fi\if\myproof0 \cite{Nghia17}. \fi
Theorem~\ref{theo:1} establishes a universal bound that holds simultaneously for all instantaneous regrets $r_t$ with an arbitrarily high confidence and is adjustable via parameter $\beta_t$ to trade off between exploitation vs. exploration. More importantly, Theorem~\ref{theo:1} immediately implies the following bound on the cumulative regret $R_n$ based on the notion of \emph{maximum information gain} below:
\begin{definition}
\label{def:1}
Let $\mathcal{A}\triangleq\{\mathbf{x}_t\}_{t=1,\ldots,n} \subseteq \mathcal{D}$ and 
$\mathbf{f}_\mathcal{A}^\mathcal{I} \triangleq (f_\mathcal{I}(\mathbf{x}^\mathcal{I}_t))^\top_{t=1,\ldots,n}$. 
Suppose that a column vector $\mathbf{y}^\mathcal{I}_\mathcal{A} \triangleq (y_\mathcal{I}(\mathbf{x}^\mathcal{I}_t))^\top_{t=1,\ldots,n}$ of noisy outputs $y_\mathcal{I}(\mathbf{x}^\mathcal{I}_t) = f_\mathcal{I}(\mathbf{x}^\mathcal{I}_t) + \epsilon$ can be observed from evaluating the latent factor function $f_\mathcal{I}$ at input queries $\mathbf{x}^\mathcal{I}_1,\ldots,\mathbf{x}^\mathcal{I}_n$, respectively.
%
Then, the maximum information gain about $\mathbf{f}^\mathcal{I}_\mathcal{A}$ given $\mathbf{y}^\mathcal{I}_\mathcal{A}$ can be characterized in terms of their Shannon mutual information:
$
\gamma_n^\mathcal{I} \triangleq\max_{\mathcal{A}:\mathcal{A}\subseteq\mathcal{D},|\mathcal{A}|=n} \boldsymbol{I}(\mathbf{f}^\mathcal{I}_\mathcal{A},\mathbf{y}_\mathcal{A}^\mathcal{I}) 
$.
The total maximum information gain is
$\gamma_n \triangleq \sum_{\mathcal{I}\in\mathcal{U}} \gamma_n^\mathcal{I}$.
\end{definition}
The total maximum information gain $\gamma_n$ defined above can then be exploited to bound the cumulative regret $R_n$:
\begin{theorem}
\label{theo:2}
Given $\delta\in(0,1)$, 
$
\mathrm{Pr}(R_n  \leq (Cn\beta_n\gamma_n)^{1/2}  \in o(n)) \geq 1 - \delta 
$
under some mild condition on $\sigma_{t}^{\mathcal{I}}(\mathbf{x}^{\mathcal{I}})^2$
where $C$ is some constant defined in\if\myproof1 Appendix~\ref{app:b}.\fi\if\myproof0 \cite{Nghia17}.\fi
\end{theorem}
Its proof is in\if\myproof1 Appendix~\ref{app:b}.\fi\if\myproof0 \cite{Nghia17}.\fi
\subsubsection{Remark}~\citeauthor{Kandasamy15}~\shortcite{Kandasamy15} have attempted to bound the cumulative regret $R_n$ in terms of the maximum information gain $\boldsymbol{I}(\mathbf{f}_\mathcal{A},\mathbf{y}_\mathcal{A})$ of the objective function $f$ directly (i.e., $\mathbf{f}_\mathcal{A}\triangleq (f(\mathbf{x}_t))^\top_{t=1,\ldots,n}$ and $\mathbf{y}_\mathcal{A}\triangleq (y_t)^\top_{t=1,\ldots,n}$) for an extreme special case of our work where the effects of all input components on the output of $f$ are statistically independent. The validity of their proof appears to depend on the assumption that the sum of the GP posterior variances of the outputs of all factor functions is less than or equal to the GP posterior variance of the output of the objective function $f$, which is flawed as recently acknowledged by the authors in \cite{Kandasamy15a}. Our analysis\if\myproof1 (Appendix~\ref{app:b}) \fi\if\myproof0 \cite{Nghia17} \fi instead uses a different quantity (see Definition $1$) to bound the cumulative regret $R_n$ and therefore avoid making such a flawed assumption.

Theorem~\ref{theo:2} implies $\lim_{n \rightarrow \infty} R_n/n = 0$ which guarantees the desired asymptotic optimality of DEC-HBO (hence, no regret) with an arbitrarily high confidence. However, when the input space is infinite, $\langle\beta_t\rangle_t$ tend to infinity and hence void the above analysis. To address this caveat, we extend our analysis to handle infinite, continuous input spaces by assuming Lipschitz continuity of the objective function $f$.
\subsubsection{Continuous Input Space.}
\label{continuous}
To extend our previous analysis to the setting with infinite, continuous input spaces, we assume  objective function $f$ to be $L$-Lipschitz continuous:
\begin{assumption} 
There exist constants $a, b, L > 0$ such that 
$\mathrm{Pr}(\forall \mathbf{x},\mathbf{x}' \in \mathcal{D}\ \ |f(\mathbf{x}) - f(\mathbf{x}')| \leq L\|\mathbf{x} - \mathbf{x}'\|_1)\geq 1- a|\mathcal{U}|\exp(-L^2/b^2)$.
\label{aaa}
\end{assumption}

\noindent
The Lipschitz continuity of $f$ 
can be exploited to establish an upper bound on the cumulative regret $R_n$ without relying on the finiteness of the input space. Intuitively, the key idea is to repeat the above finite-case analysis
for a finite discretization of the input space
by first establishing regret bounds for these discretized inputs.
The resulting bounds can then be related to an arbitrary input  by using the Lipschitz continuity of $f$ in Assumption 3 to bound its output in terms of that of its closest discretized input  with high probability. As such, it can be shown 
that if the discretization (especially its granularity) is carefully designed, then the cumulative regret $R_n$ of our DEC-HBO algorithm can be bounded in terms of the above Lipschitz constants as well as the discretization granularity instead of the (infinite) size of the input space:
\begin{theorem}
\label{theo:3}
Given $\delta\in(0,1)$, there exists a monotonically increasing sequence $\langle\beta_t\rangle_t$ such that $\beta_t \in \mathcal{O}(\log t)$ and
$\mathrm{Pr}(R_n \leq (Cn\beta_n\gamma_n)^{1/2} + \pi^2/6 \in\ o(n)) \geq 1 - \delta$
under some mild condition on $\sigma_{t}^{\mathcal{I}}(\mathbf{x}^{\mathcal{I}})^2$.
\end{theorem}
Its proof is in\if\myproof1 Appendix~\ref{app:c}. \fi\if\myproof0 \cite{Nghia17}. \fi 
Theorem~\ref{theo:3} concludes our analysis for the infinite case which is similar to Theorem~\ref{theo:2} for the finite case: The upper bound on the cumulative regret $R_n$ is sub-linear in $n$, which guarantees that its average regret approaches zero in the limit. So,  DEC-HBO is asymptotically optimal with an arbitrarily high confidence. 
%
\section{Experiments and Discussion}
\label{exp}
This section empirically evaluates the performance of our DEC-HBO algorithm on an extensive benchmark comprising three synthetic functions: Shekel ($4$-dimensional), Hartmann ($6$-dimensional), Michalewicz ($10$-dimensional) (Section \ref{synthetic}), and two high-dimensional optimization problems involving hyperparameter tuning for popular machine learning models such as sparse GP \cite{Snelson07a} and convolutional neural network 
modeling two real-world datasets: 
Physicochemical properties of protein tertiary structure \cite{UCI_protein_data} and CIFAR-10 (Section \ref{hypertuning}).
%
\subsection{Optimizing Synthetic Functions}
\label{synthetic}
This section empirically compares the performance of DEC-HBO (with maximum factor size of $2$ (MF2) or $3$ (MF3)) with that of the state-of-the-art HBO algorithms like ADD-GP-UCB \cite{Kandasamy15}, ADD-MES-G and ADD-MES-R \cite{Wang17b}, and REMBO~\cite{Wang13} in optimizing the Shekel, Hartmann, and Michalewicz functions\if\myproof1 (Appendix~\ref{app:d}).\fi\if\myproof0 \cite{Nghia17}.\fi
\begin{figure}
	\begin{tabular}{cc}
		\includegraphics[width=3.79cm]{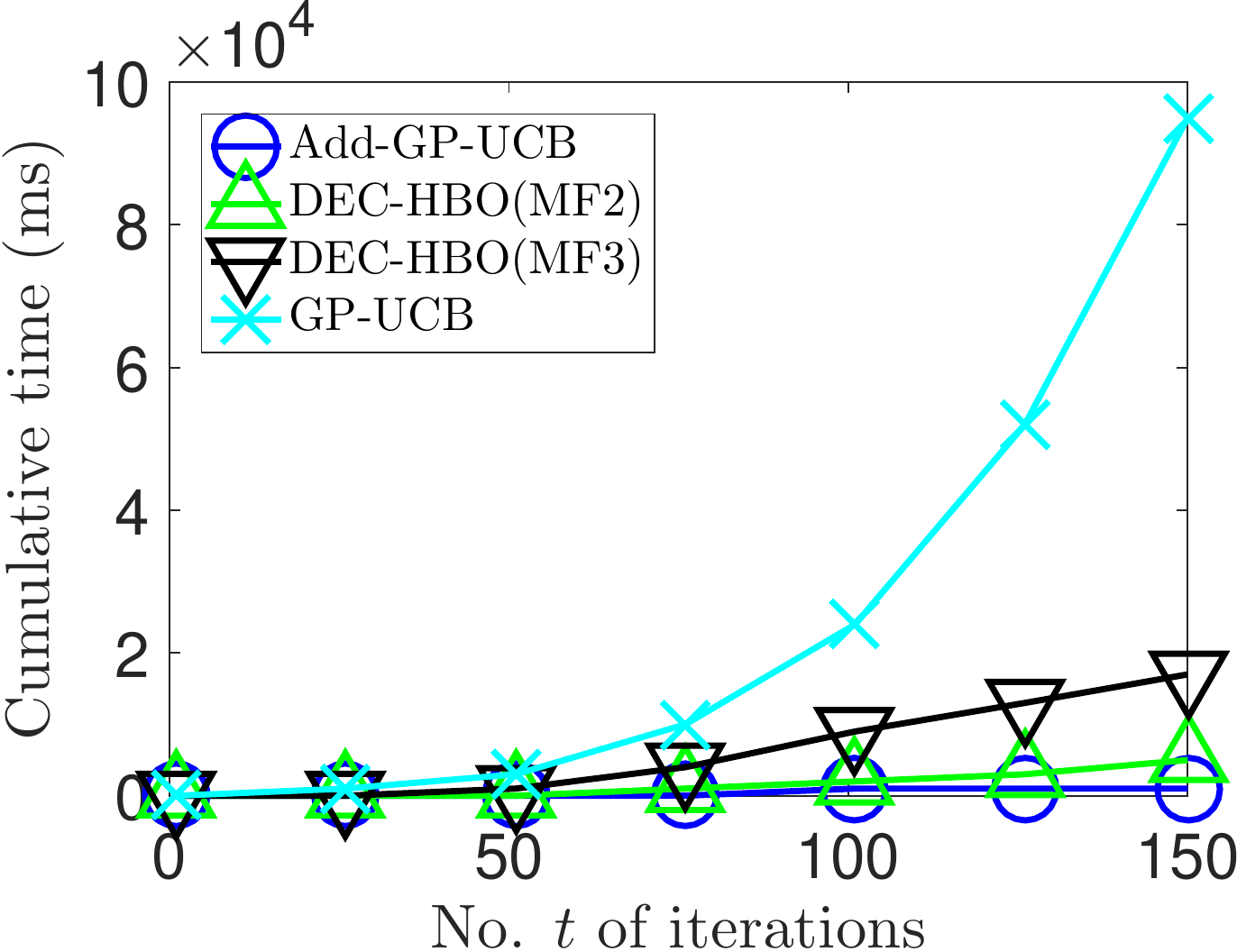} &
		\includegraphics[width=3.79cm]{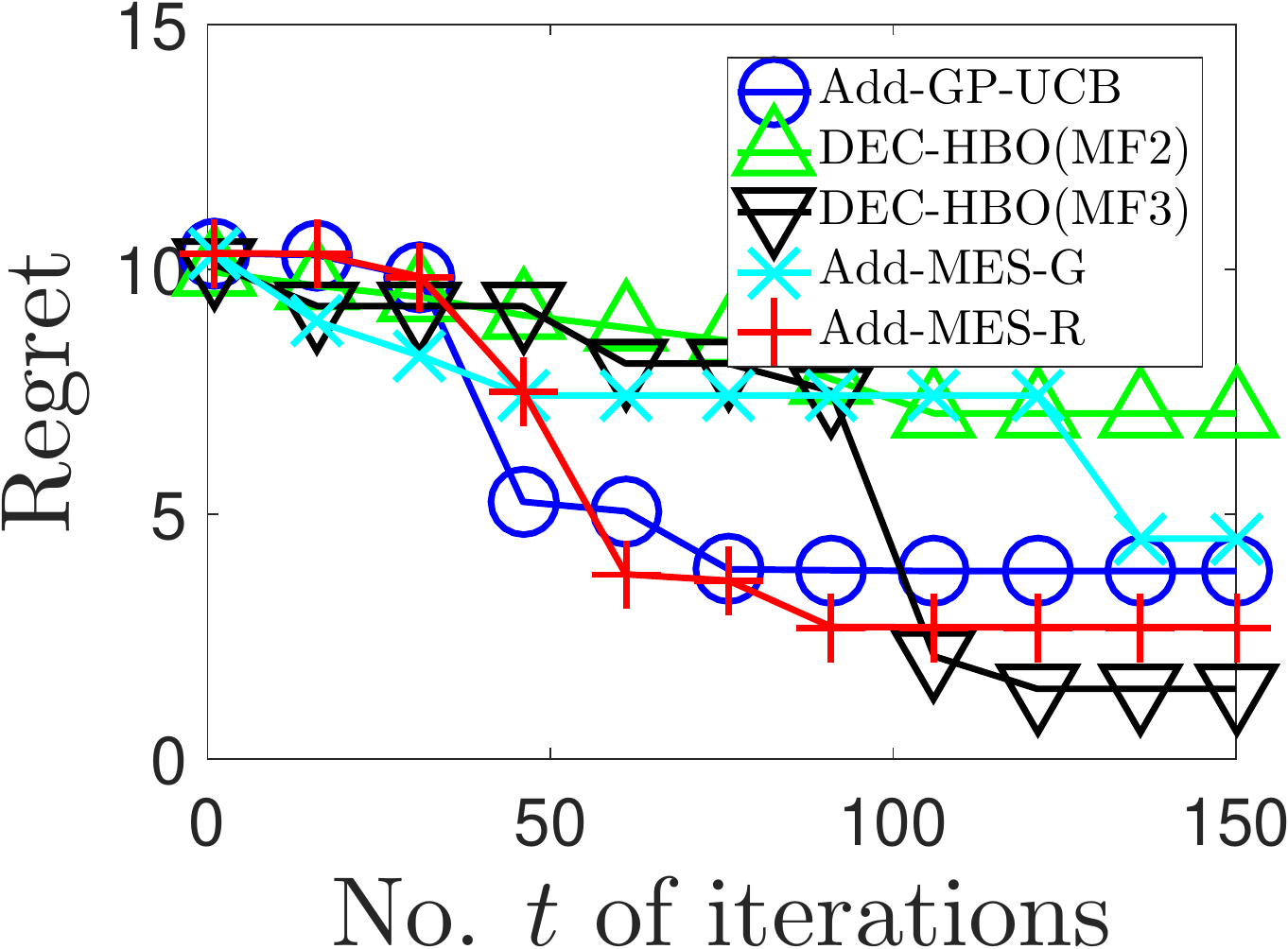} \\
		(a) & (b) \\
		\includegraphics[width=3.79cm]{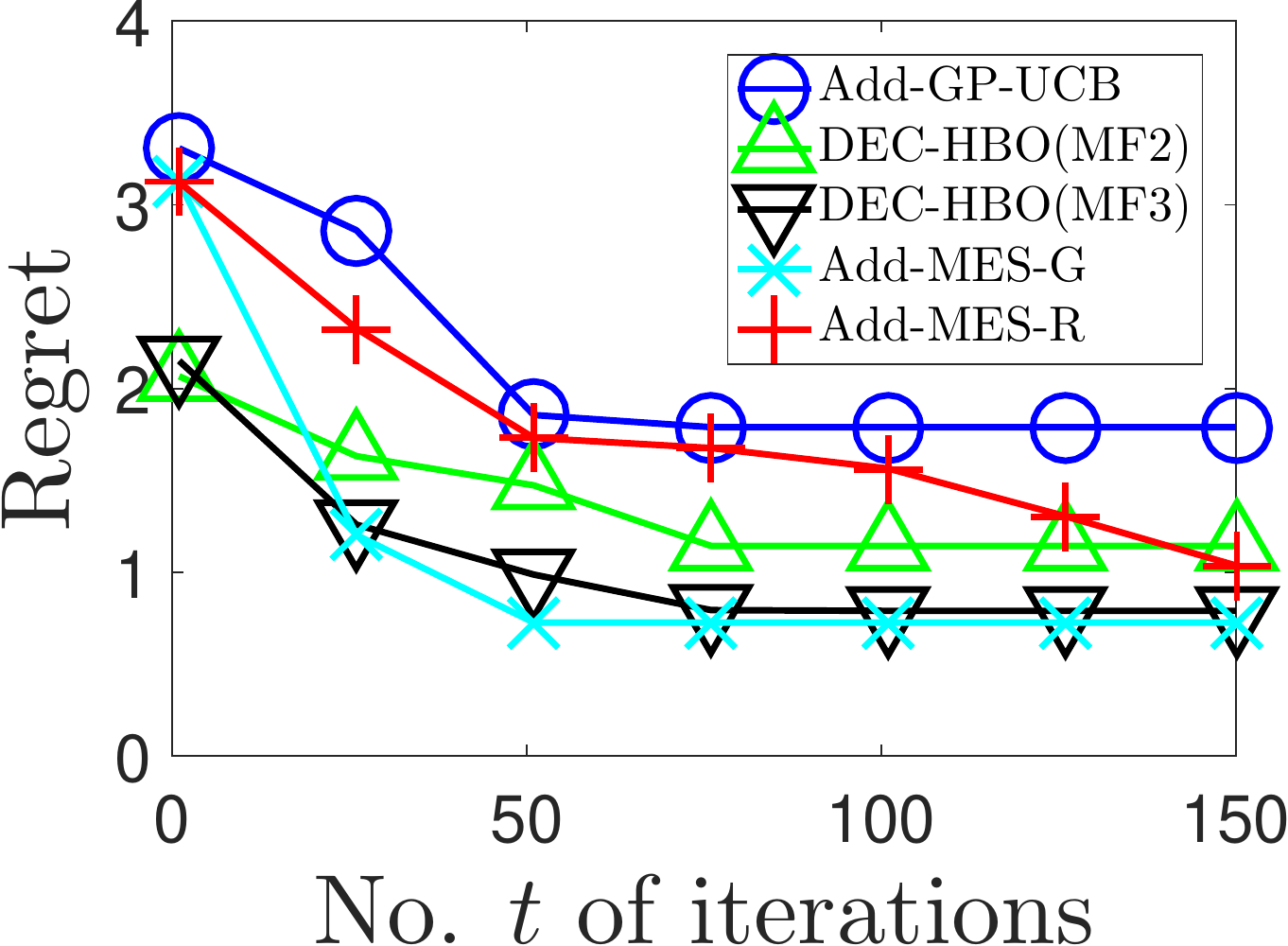} &
		\includegraphics[width=3.79cm]{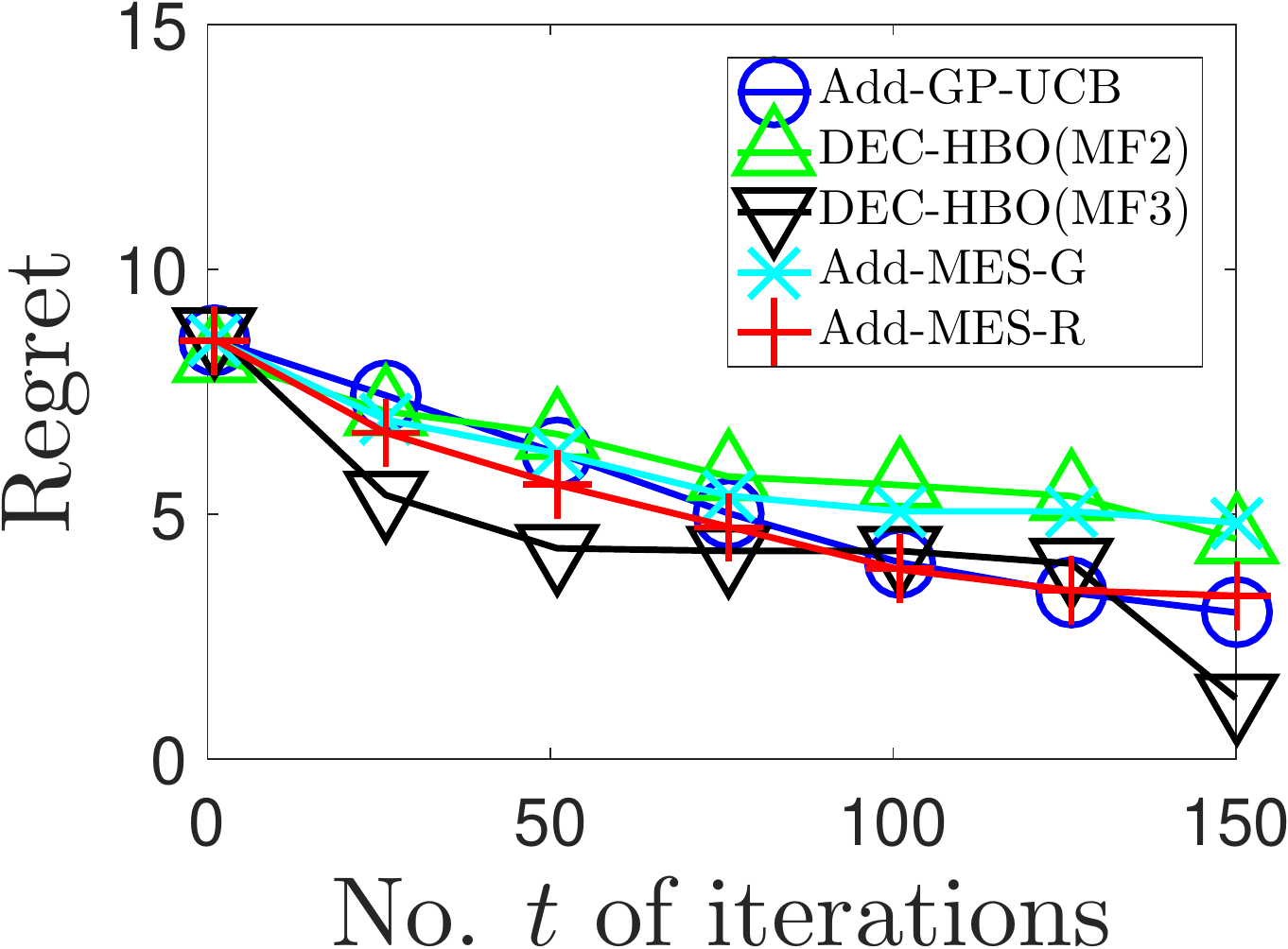} \\
		(c) & (d) 
	\end{tabular}
	\caption{(a) Graphs of cumulative incurred time of tested algorithms vs. no. $t$ of BO iterations for Shekel function, graphs of regret (i.e., $\min_{i=1}^t f(\mathbf{x}_i) - f(\mathbf{x}_\ast)$)
	achieved by tested HBO algorithms vs. no. $t$ of BO iterations for (b) Shekel, (c) Hartmann, and (d) Michalewicz functions.}
	\label{fig:synthetic}
\end{figure}

Fig. 1 reports results of the regret 
(i.e., $\min_{i=1}^t f(\mathbf{x}_i) - f(\mathbf{x}_\ast)$)
of the tested algorithms averaged over $5$ runs, each of which comprises $150$ iterations. For DEC-HBO (MF2) and DEC-HBO (MF3), each BO iteration involves $30$ iterations of max-sum.
For clarity, Table 1 further reports the final converged regrets achieved by the tested HBO algorithms  including REMBO\footnote{The performance of REMBO is not plotted in Figs. 1 and 2 to ease clutter as it requires much more iterations for convergence using the authors' implementation: \url{github.com/ziyuw/rembo}.}. The results show that in general, our DEC-HBO variants perform competitively with the other state-of-the-art HBO algorithms for all synthetic functions. Interestingly, it can also be observed that DEC-HBO (MF3) consistently outperforms DEC-HBO (MF2) and Add-GP-UCB (corresponding to the DEC-HBO variant with max. factor size of $1$) for all synthetic functions with input dimension $d \geq 4$. This highlights the importance of exploiting the interdependent effects of various input components on the output of $f$. In most cases, DEC-HBO (MF3) also outperforms Add-MES-G, Add-MES-R, and REMBO with the difference being most pronounced for the $10$-dimensional Michalewicz function. This further indicates the efficacy of DEC-HBO when applied to higher-dimensional optimization problems and asserts that preserving the interdependent effects of various input components on the output of $f$ is necessary. 
Fig.~\ref{fig:synthetic}a also shows the cumulative incurred time of GP-UCB, Add-GP-UCB, DEC-HBO (MF2), and DEC-HBO (MF3) in optimizing the Shekel function. The results reveal that DEC-HBO performs competitively in terms of time efficiency with Add-GP-UCB with minimal increase in incurred time over $150$ iterations in exchange for a significant improvement in terms of BO performance. In contrast, GP-UCB incurs much more time than our DEC-HBO variants, thus asserting the computational advantage of our decentralized optimization algorithm.
\begin{table}
	\centering
		\begin{tabular}{l|ccc}
			\textbf{HBO} & \textbf{Hartmann} & \textbf{Shekel} & \textbf{Michalewicz} \\ \hline 
			\textbf{MF2} & 1.1436 & 7.0538 & 4.4944 \\ 
			\textbf{MF3} & \textbf{0.7904} & \textbf{1.4295} & \textbf{1.2367} \\ 
			{Add-GP-UCB} & 1.7898 & 3.8338 & 2.9870 \\ 
			{Add-MES-G} & \textbf{0.7268} & 4.4951 & 4.8227 \\ 
			{Add-MES-R} & 1.0372 & 2.6858 & 3.3296 \\  
			{REMBO} & 1.5843 & 5.1677 & 4.0524 
		\end{tabular}
		\caption{Regrets achieved by tested HBO algorithms for Hartmann, Shekel, and Michalewicz functions.}
\label{table:lowdim}	
\end{table}
\subsection{Optimizing Hyperparameters of ML Models}
\label{hypertuning}
This section demonstrates the effectiveness of DEC-HBO in tuning the hyperparameters of two ML models like the sparse \emph{partially independent conditional} (PIC) approximation of GP model  \cite{Snelson07a} and \emph{convolutional neural network} (CNN).
The goal is to find the optimal configuration of (a) kernel hyperparameters and inducing inputs for which PIC predicts well for the physicochemical properties of protein tertiary structure dataset~\cite{UCI_protein_data} and (b) network hyperparameters for which CNN classifies well for the  CIFAR-10 dataset. 
These PIC and CNN hyperparameter tuning tasks are detailed as follows:

\subsubsection{PIC.} The PIC model is trained using the physicochemical properties of protein tertiary structure dataset~\cite{UCI_protein_data} which has $45730$ instances, each of which contains $\kappa=9$ attributes describing the physicochemical properties of a protein residue and its size (in armstrong) to be predicted. 
$95\%$ and $5\%$ of the dataset are used as training and test data, respectively. 
The training data is further divided into $5$ equal folds. The goal is to find a hyperparameter configuration that minimizes the \emph{root mean square error} (RMSE) of PIC's prediction on the test data. This is achieved via BO using the $5$-fold validation performance as a noisy estimate of the real performance on the test data. Specifically, for every input query of hyperparameters, the corresponding PIC model separately predicts on each of these folds (validation data), having trained on the remaining folds (effective training data). The averaged prediction error over these $5$ folds is then returned to the HBO algorithm to update the acquisition function for selecting the next input query of hyperparameters. Every such input query contains $2+ \kappa + \nu \times \kappa$ hyperparameters which include the signal and noise variances, $\kappa$ length-scales of the squared exponential kernel, and $\nu=200$ inducing inputs of dimension $\kappa$ each.
%
\subsubsection{CNN.} The CNN model is trained using the CIFAR-10 
object recognition dataset which has $50000$ training images and $10000$ test images, each of which belongs to one of the ten classes. $5000$ training images are set aside as the validation data. Similar to PIC, the goal is to find a hyperparameter configuration that minimizes the classification error of CNN on the test data, which is likewise achieved via BO using the performance on the validation data to estimate the real performance on the test data\footnote{We use the same CNN structure as the example code of keras: \url{github.com/fchollet/keras/} and replace the default optimizer in their code by \emph{stochastic gradient descent} (SGD).}. The six CNN hyperparameters to be optimized in our experiments include the learning rate of SGD in the range of $[10^{-5},1]$, three dropout rates in the range of $[0, 1]$, batch size in the range of $[100, 1000]$, and number of learning epochs in the range of $[100, 1000]$.

Fig. 2 shows results of the performance of DEC-HBO variants in comparison to that of ADD-GP-UCB for hyperparameter tuning of PIC and CNN trained with real-world datasets. 
Table 2 further reports the final converged RMSE achieved by the tested HBO algorithms including REMBO for PIC hyperparameter tuning\footnote{The performance of DEC-HBO is not compared with that of REMBO (implemented in MATLAB) for CNN hyperparameter tuning as the CNN code in keras cannot be converted to MATLAB.}. 
It can be observed that in general, our DEC-HBO variants outperform ADD-GP-UCB and REMBO. Interestingly, for PIC hyperparameter tuning, the performance difference is also more pronounced in the early BO iterations, which suggests that DEC-HBO excels in time-constrained high-dimensional optimization problems and preserving the interdependent effects of various input components on the output of $f$ boosts the performance of GP-UCB-based algorithms. This is consistent with our earlier observations in Section~\ref{synthetic}. The poor performance of REMBO as compared to DEC-HBO is expected since it only considers input hyperparameter queries generated from a random low-dimensional embedding of the input space, which severely restricts the expressiveness of PIC model.
\begin{figure}
	\begin{tabular}{cc}
		\includegraphics[width=3.79cm]{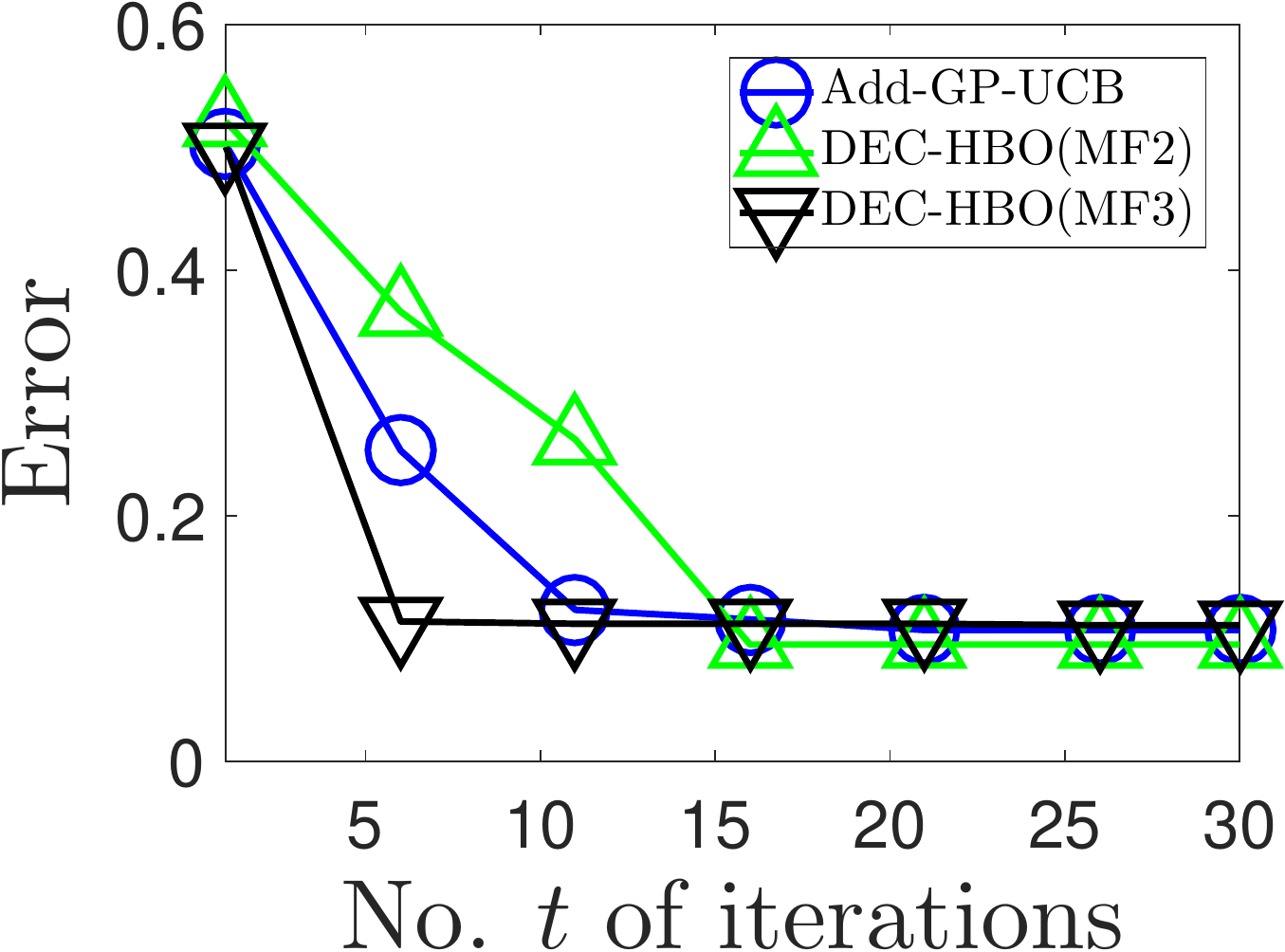} &
		\includegraphics[width=3.79cm]{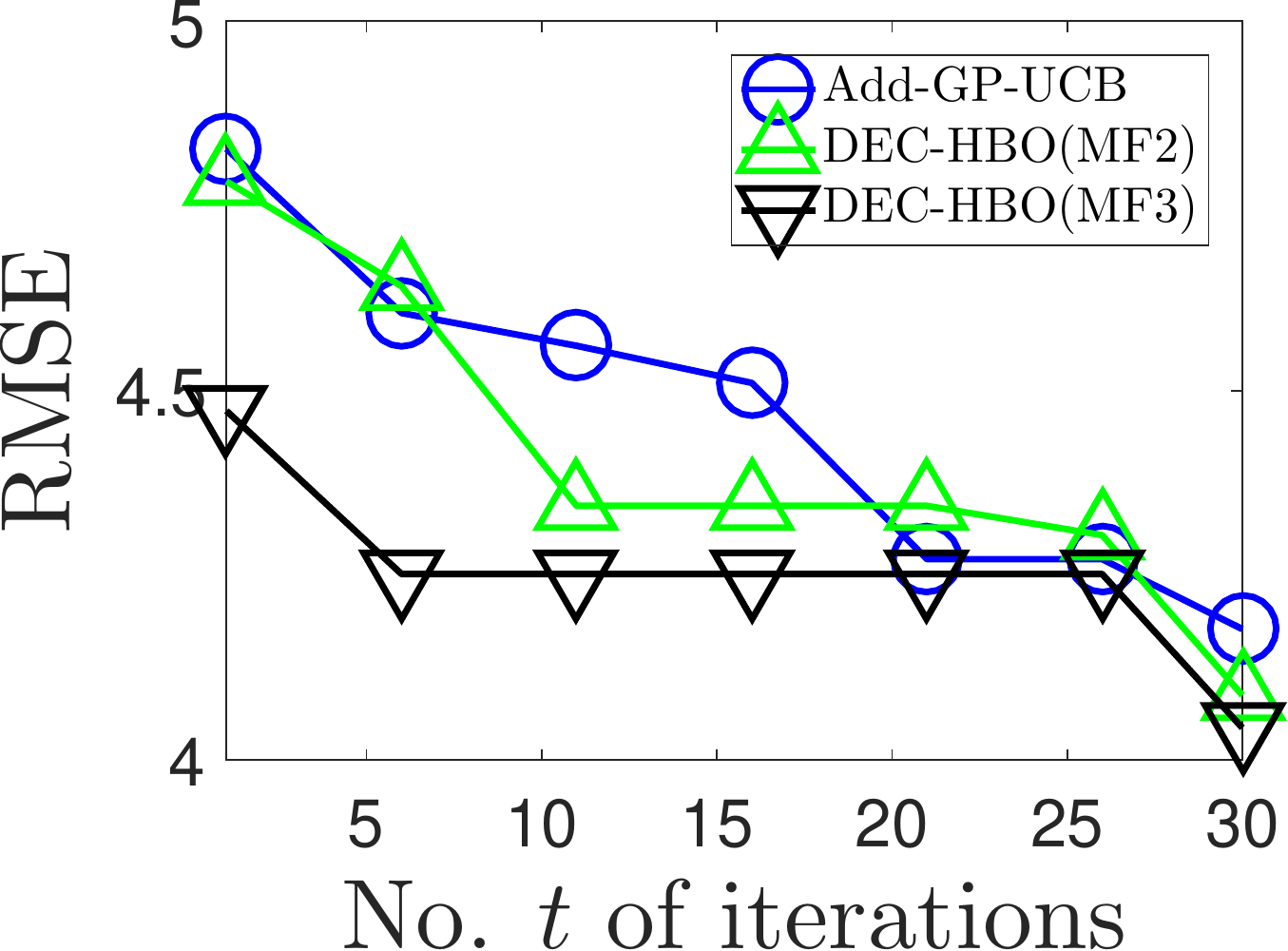} \\
		(a) & (b)
	\end{tabular}
	\caption{Graphs of (a) classification error of CNN and (b) RMSE of PIC's prediction vs. no. $t$ of BO iterations.}
	\label{fig:tuning}
\end{figure}
\section{Conclusion}
This paper describes a novel DEC-HBO algorithm that, in contrast to existing HBO algorithms, can exploit the interdependent effects of various input components on the output of the unknown objective function $f$ for boosting the BO performance and still preserve scalability in the number of input dimensions as well as guarantee no-regret performance asymptotically (see Remark in Section~\ref{discrete}).
To realize this, we propose a sparse yet rich factor graph representation of $f$ to be exploited for designing an acquisition function that can be similarly represented by a sparse factor graph and hence be efficiently optimized in a decentralized manner using a class of distributed message passing algorithms.
%
Empirical evaluation on both synthetic and real-world experiments show that our DEC-HBO algorithm performs competitively to the state-of-the-art centralized BO and HBO algorithms while providing a significant computational advantage for high-dimensional optimization problems.
For future work, we plan to generalize DEC-HBO to batch mode~\cite{Erik17} and the nonmyopic context by appealing to existing literature on nonmyopic BO~\cite{ling16} and active learning~\cite{LowAAMAS13,LowECML14b,NghiaICML14,LowAAMAS08,LowICAPS09,LowAAMAS11} 
as well as to be performed by a multi-robot team to find hotspots in environmental sensing/monitoring by seeking inspiration from existing literature on multi-robot active sensing/learning~\cite{LowRSS13,LowUAI12,LowTASE15,LowAAMAS12,LowAAMAS14}.
For applications with a huge budget of function evaluations, we like to couple DEC-HBO with the use of parallel/distributed~\cite{Chen13,HoangICML16,low15} and online/stochastic~\cite{NghiaICML15,MinhAAAI17,Xu2014} sparse GP models
to represent the belief of $f$ efficiently.
\begin{table}
\centering
	\begin{tabular}{l|cccc}
		\textbf{HBO}     & \textbf{MF2} & \textbf{MF3} & \textbf{Add-GP-UCB} & \textbf{REMBO}  \\ \hline 
		\textbf{PIC} & 4.0879       & 4.0437       & 4.1774 & 9.9100 \\  
		\textbf{CNN} & 0.0950        & 0.1107      & 0.1068 & - 
	\end{tabular}
	\caption{Minimum errors achieved by tested HBO algorithms for hyperparameter tuning of PIC and CNN.}
	\label{table:highdim}	
\end{table}

\subsubsection{Acknowledgments.} This research is supported by the National Research Foundation, Prime Minister's Office, Singapore under its Campus for Research Excellence and Technological Enterprise (CREATE) programme.

\bibliographystyle{aaai}
\bibliography{aaai18}

\begin{thebibliography}{}

\bibitem[\protect\citeauthoryear{Bergstra, Yamins, and Cox}{2013}]{Yamins13}
Bergstra, J.; Yamins, D.; and Cox, D.~D.
\newblock 2013.
\newblock Making a science of model search: Hyperparameter optimization and
  hundreds of dimensions for vision architectures.
\newblock In {\em Proc. {ICML}},  115--123.

\bibitem[\protect\citeauthoryear{Cao, Low, and Dolan}{2013}]{LowAAMAS13}
Cao, N.; Low, K.~H.; and Dolan, J.~M.
\newblock 2013.
\newblock Multi-robot informative path planning for active sensing of
  environmental phenomena: A tale of two algorithms.
\newblock In {\em Proc. {AAMAS}}.

\bibitem[\protect\citeauthoryear{Chen \bgroup et al\mbox.\egroup
  }{2012}]{LowUAI12}
Chen, J.; Low, K.~H.; Tan, C. K.-Y.; Oran, A.; Jaillet, P.; Dolan, J.~M.; and
  Sukhatme, G.~S.
\newblock 2012.
\newblock Decentralized data fusion and active sensing with mobile sensors for
  modeling and predicting spatiotemporal traffic phenomena.
\newblock In {\em Proc. UAI},  163--173.

\bibitem[\protect\citeauthoryear{Chen \bgroup et al\mbox.\egroup
  }{2013}]{Chen13}
Chen, J.; Cao, N.; Low, K.~H.; Ouyang, R.; Tan, C. K.-Y.; and Jaillet, P.
\newblock 2013.
\newblock Parallel {G}aussian process regression with low-rank covariance
  matrix approximations.
\newblock In {\em Proc. UAI},  152--161.

\bibitem[\protect\citeauthoryear{Chen \bgroup et al\mbox.\egroup
  }{2015}]{LowTASE15}
Chen, J.; Low, K.~H.; Jaillet, P.; and Yao, Y.
\newblock 2015.
\newblock Gaussian process decentralized data fusion and active sensing for
  spatiotemporal traffic modeling and prediction in mobility-on-demand systems.
\newblock {\em {IEEE} Trans. Autom. Sci. Eng.} 12:901--921.

\bibitem[\protect\citeauthoryear{Chen, Low, and Tan}{2013}]{LowRSS13}
Chen, J.; Low, K.~H.; and Tan, C. K.-Y.
\newblock 2013.
\newblock {Gaussian} process-based decentralized data fusion and active sensing
  for mobility-on-demand system.
\newblock In {\em Proc. {RSS}}.

\bibitem[\protect\citeauthoryear{Daxberger and Low}{2017}]{Erik17}
Daxberger, E., and Low, K.~H.
\newblock 2017.
\newblock Distributed batch {Gaussian} process optimization.
\newblock In {\em Proc. {ICML}},  951--960.

\bibitem[\protect\citeauthoryear{Djolonga, Krause, and
  Cevher}{2013}]{Djolonga13}
Djolonga, J.; Krause, A.; and Cevher, V.
\newblock 2013.
\newblock High-dimensional {G}aussian process bandits.
\newblock In {\em Proc. {NIPS}}.

\bibitem[\protect\citeauthoryear{Duvenaud, Nickisch, and
  Rasmussen}{2011}]{Duvenaud11}
Duvenaud, D.~K.; Nickisch, H.; and Rasmussen, C.~E.
\newblock 2011.
\newblock Additive {Gaussian} processes.
\newblock In {\em Proc. {NIPS}},  226--234.

\bibitem[\protect\citeauthoryear{Gardner \bgroup et al\mbox.\egroup
  }{2017}]{Garnett17}
Gardner, J.~R.; Guo, C.; Weinberger, K.~Q.; Garnett, R.; and Grosse, R.
\newblock 2017.
\newblock Discovering and exploiting additive structure for {Bayesian}
  optimization.
\newblock In {\em Proc. {AISTATS}}.

\bibitem[\protect\citeauthoryear{Gonz{\'{a}}lez \bgroup et al\mbox.\egroup
  }{2014}]{Gonzalez14}
Gonz{\'{a}}lez, J.; Longworth, J.; James, D.; and Lawrence, N.
\newblock 2014.
\newblock Bayesian optimization for synthetic gene design.
\newblock In {\em NIPS Workshop on {B}ayesian Optimization in Academia and
  Industry}.

\bibitem[\protect\citeauthoryear{Hennig and Schuler}{2012}]{Hennig12}
Hennig, P., and Schuler, C.~J.
\newblock 2012.
\newblock Entropy search for information-efficient global optimization.
\newblock {\em JMLR} 13:1809--1837.

\bibitem[\protect\citeauthoryear{Hoang \bgroup et al\mbox.\egroup
  }{2014a}]{LowECML14b}
Hoang, T.~N.; Low, K.~H.; Jaillet, P.; and Kankanhalli, M.
\newblock 2014a.
\newblock Active learning is planning: Nonmyopic $\epsilon$-{Bayes}-optimal
  active learning of {Gaussian} processes.
\newblock In {\em Proc. {ECML/PKDD Nectar Track}},  494--498.

\bibitem[\protect\citeauthoryear{Hoang \bgroup et al\mbox.\egroup
  }{2014b}]{NghiaICML14}
Hoang, T.~N.; Low, K.~H.; Jaillet, P.; and Kankanhalli, M.
\newblock 2014b.
\newblock Nonmyopic $\epsilon$-{B}ayes-optimal active learning of {Gaussian}
  processes.
\newblock In {\em Proc. ICML},  739--747.

\bibitem[\protect\citeauthoryear{Hoang, Hoang, and Low}{2015}]{NghiaICML15}
Hoang, T.~N.; Hoang, Q.~M.; and Low, K.~H.
\newblock 2015.
\newblock A unifying framework of anytime sparse {Gaussian} process regression
  models with stochastic variational inference for big data.
\newblock In {\em Proc. {ICML}},  569--578.

\bibitem[\protect\citeauthoryear{Hoang, Hoang, and Low}{2016}]{HoangICML16}
Hoang, T.~N.; Hoang, Q.~M.; and Low, K.~H.
\newblock 2016.
\newblock A distributed variational inference framework for unifying parallel
  sparse {Gaussian} process regression models.
\newblock In {\em Proc. ICML},  382--391.

\bibitem[\protect\citeauthoryear{Hoang, Hoang, and Low}{2017}]{MinhAAAI17}
Hoang, Q.~M.; Hoang, T.~N.; and Low, K.~H.
\newblock 2017.
\newblock A generalized stochastic variational {Bayesian} hyperparameter
  learning framework for sparse spectrum {Gaussian} process regression.
\newblock In {\em Proc. {AAAI}},  2007--2014.

\bibitem[\protect\citeauthoryear{Hornby \bgroup et al\mbox.\egroup
  }{2006}]{Hornby06}
Hornby, G.~S.; Globus, A.; Linden, D.~S.; and Lohn, J.~D.
\newblock 2006.
\newblock Automated antenna design with evolutionary algorithms.
\newblock In {\em Proc. {AIAA} Space Conference}.

\bibitem[\protect\citeauthoryear{Kandasamy, Schneider, and
  P\'{o}czos}{2015}]{Kandasamy15}
Kandasamy, K.; Schneider, J.; and P\'{o}czos, B.
\newblock 2015.
\newblock High dimensional {B}ayesian optimization and bandits via additive
  models.
\newblock In {\em Proc. {ICML}}.

\bibitem[\protect\citeauthoryear{Kandasamy, Schneider, and
  P\'{o}czos}{2016}]{Kandasamy15a}
Kandasamy, K.; Schneider, J.; and P\'{o}czos, B.
\newblock 2016.
\newblock High-dimensional {B}ayesian optimization and bandits via additive
  models.
\newblock {arXiv:1503.01673}.

\bibitem[\protect\citeauthoryear{Kr{\"{a}}henb{\"{u}}hl and
  Koltun}{2011}]{Krahenbuhl11}
Kr{\"{a}}henb{\"{u}}hl, P., and Koltun, V.
\newblock 2011.
\newblock Efficient inference in fully connected {CRFs} with {Gaussian} edge
  potentials.
\newblock In {\em Proc. {NIPS}}.

\bibitem[\protect\citeauthoryear{Leite, Enembreck, and
  Barth{\`{e}}s}{2014}]{Leite14}
Leite, A.~R.; Enembreck, F.; and Barth{\`{e}}s, J.-P.~A.
\newblock 2014.
\newblock Distributed constraint optimization problems: Review and
  perspectives.
\newblock {\em Expert Systems with Applications} 41:5139--5157.

\bibitem[\protect\citeauthoryear{Li \bgroup et al\mbox.\egroup }{2016}]{Li16}
Li, C.-L.; Kandasamy, K.; P\'{o}czos, B.; and Schneider, J.
\newblock 2016.
\newblock High dimensional {Bayesian} optimization via restricted projection
  pursuit models.
\newblock In {\em Proc. {AISTATS}}.

\bibitem[\protect\citeauthoryear{Ling, Low, and Jaillet}{2016}]{ling16}
Ling, C.~K.; Low, K.~H.; and Jaillet, P.
\newblock 2016.
\newblock {Gaussian} process planning with {Lipschitz} continuous reward
  functions: Towards unifying {Bayesian} optimization, active learning, and
  beyond.
\newblock In {\em Proc. {AAAI}},  1860--1866.

\bibitem[\protect\citeauthoryear{Low \bgroup et al\mbox.\egroup
  }{2012}]{LowAAMAS12}
Low, K.~H.; Chen, J.; Dolan, J.~M.; Chien, S.; and Thompson, D.~R.
\newblock 2012.
\newblock Decentralized active robotic exploration and mapping for
  probabilistic field classification in environmental sensing.
\newblock In {\em Proc. {AAMAS}},  105--112.

\bibitem[\protect\citeauthoryear{Low \bgroup et al\mbox.\egroup }{2015}]{low15}
Low, K.~H.; Yu, J.; Chen, J.; and Jaillet, P.
\newblock 2015.
\newblock Parallel {Gaussian} process regression for big data: Low-rank
  representation meets {Markov} approximation.
\newblock In {\em Proc. {AAAI}}.

\bibitem[\protect\citeauthoryear{Low, Dolan, and Khosla}{2008}]{LowAAMAS08}
Low, K.~H.; Dolan, J.~M.; and Khosla, P.
\newblock 2008.
\newblock Adaptive multi-robot wide-area exploration and mapping.
\newblock In {\em Proc. {AAMAS}},  23--30.

\bibitem[\protect\citeauthoryear{Low, Dolan, and Khosla}{2009}]{LowICAPS09}
Low, K.~H.; Dolan, J.~M.; and Khosla, P.
\newblock 2009.
\newblock Information-theoretic approach to efficient adaptive path planning
  for mobile robotic environmental sensing.
\newblock In {\em Proc. {ICAPS}}.

\bibitem[\protect\citeauthoryear{Low, Dolan, and Khosla}{2011}]{LowAAMAS11}
Low, K.~H.; Dolan, J.~M.; and Khosla, P.
\newblock 2011.
\newblock Active {Markov} information-theoretic path planning for robotic
  environmental sensing.
\newblock In {\em Proc. {AAMAS}},  753--760.

\bibitem[\protect\citeauthoryear{Naver{\v{s}}nik and Rojnik}{2012}]{Klemen12}
Naver{\v{s}}nik, K., and Rojnik, K.
\newblock 2012.
\newblock Handling input correlations in pharmacoeconomic models.
\newblock {\em Value in Health} 15:540--549.

\bibitem[\protect\citeauthoryear{Ouyang \bgroup et al\mbox.\egroup
  }{2014}]{LowAAMAS14}
Ouyang, R.; Low, K.~H.; Chen, J.; and Jaillet, P.
\newblock 2014.
\newblock Multi-robot active sensing of non-stationary {Gaussian} process-based
  environmental phenomena.
\newblock In {\em Proc. {AAMAS}}.

\bibitem[\protect\citeauthoryear{Rana}{2013}]{UCI_protein_data}
Rana, P.~S.
\newblock 2013.
\newblock Physicochemical properties of protein tertiary structure dataset.
\newblock \url{http://archive.ics.uci.edu/ml/datasets/}.

\bibitem[\protect\citeauthoryear{Rogers \bgroup et al\mbox.\egroup
  }{2011}]{Rogers11}
Rogers, A.; Farinelli, A.; Stranders, R.; and Jennings, N.~R.
\newblock 2011.
\newblock Bounded approximate decentralised coordination via the max-sum
  algorithm.
\newblock {\em AIJ} 175(2):730--759.

\bibitem[\protect\citeauthoryear{Shahriari \bgroup et al\mbox.\egroup
  }{2016}]{shahriari16}
Shahriari, B.; Swersky, K.; Wang, Z.; Adams, R.; and {de Freitas}, N.
\newblock 2016.
\newblock Taking the human out of the loop: A review of {Bayesian}
  optimization.
\newblock {\em Proceedings of the IEEE} 104(1):148--175.

\bibitem[\protect\citeauthoryear{Snelson and Ghahramani}{2007}]{Snelson07a}
Snelson, E.~L., and Ghahramani, Z.
\newblock 2007.
\newblock Local and global sparse {G}aussian process approximation.
\newblock In {\em Proc. AISTATS}.

\bibitem[\protect\citeauthoryear{Snoek, Hugo, and Adams}{2012}]{Snoek12}
Snoek, J.; Hugo, L.; and Adams, R.~P.
\newblock 2012.
\newblock Practical {B}ayesian optimization of machine learning algorithms.
\newblock In {\em Proc. {NIPS}},  2960--2968.

\bibitem[\protect\citeauthoryear{Srinivas \bgroup et al\mbox.\egroup
  }{2010}]{Srinivas10}
Srinivas, N.; Krause, A.; Kakade, S.; and Seeger, M.
\newblock 2010.
\newblock {G}aussian process optimization in the bandit setting: No regret and
  experimental design.
\newblock In {\em Proc. {ICML}},  1015--1022.

\bibitem[\protect\citeauthoryear{Wang and Jegelka}{2017}]{Wang17b}
Wang, Z., and Jegelka, S.
\newblock 2017.
\newblock Max-value entropy search for efficient {Bayesian} optimization.
\newblock In {\em Proc. {ICML}}.

\bibitem[\protect\citeauthoryear{Wang \bgroup et al\mbox.\egroup
  }{2013}]{Wang13}
Wang, Z.; Zoghi, M.; Hutter, F.; Matheson, D.; and {de Freitas}, N.
\newblock 2013.
\newblock {B}ayesian optimization in high dimensions via random embeddings.
\newblock In {\em Proc. {IJCAI}},  1778--1784.

\bibitem[\protect\citeauthoryear{Wang \bgroup et al\mbox.\egroup }{2017}]{Zi17}
Wang, Z.; Li, C.; Jegelka, S.; and Kohli, P.
\newblock 2017.
\newblock Batched high-dimensional {B}ayesian optimization via structural
  kernel learning.
\newblock In {\em Proc. {ICML}},  3656--3664.

\bibitem[\protect\citeauthoryear{Xu \bgroup et al\mbox.\egroup }{2014}]{Xu2014}
Xu, N.; Low, K.~H.; Chen, J.; Lim, K.~K.; and Ozgul, E.~B.
\newblock 2014.
\newblock {GP-Localize}: Persistent mobile robot localization using online
  sparse {Gaussian} process observation model.
\newblock In {\em Proc. AAAI},  2585--2592.

\end{thebibliography}

\if \myproof1

\appendix
\section{Time Complexity Analysis}
\label{app:e}
Let the numbers of BO iterations and decentralized optimization iterations executed by our DEC-HBO algorithm be $N_o$ and $N_m$, respectively. In BO iteration $t$, there are $|\set{U}|$ factor nodes and $d$ variable nodes operating independently. The time complexity per BO iteration for each type of node is detailed below:
\subsubsection{Factor Node.} In BO iteration $t$, the input space of factor $\mathcal{I}$ is discretized into a grid of size $\mathcal{D}_t^{|\mathcal{I}|}$. The corresponding factor function $\varphi^{\mathcal{I}}_t(\mathbf{x}^{\mathcal{I}})$~\eqref{eq:11} then needs to be evaluated over $\mathcal{D}_t^{|\mathcal{I}|}$ discretized inputs $\mathbf{x}^{\mathcal{I}}$. For each input $\mathbf{x}^{\mathcal{I}}$, the time complexity is $\mathcal{O}(t^3 |\mathcal{I}|)$ since the number of past function evaluations is $\mathcal{O}(t)$ (i.e., $t_0 + t - 1$ where $t_0$ is a constant number of function evaluations prior to running the BO algorithm). Thus, precomputing $\varphi^{\mathcal{I}}_t(\mathbf{x}^{\mathcal{I}})$ for all inputs $\mathbf{x}^\mathcal{I}$ incurs a total of $\mathcal{O}(\mathcal{D}_t^{|\mathcal{I}|}  t^3  |\mathcal{I}|)$ time. Once the precomputation is cached,  computing each message $m_{\varphi_t^{\mathcal{I}} \rightarrow \mathbf{x}^{(i)}}(h)$ incurs $\mathcal{O}(|\mathcal{I}| \mathcal{D}_t^{|\mathcal{I}| - 1})$ time~\eqref{eq:12}. Since there are $|\mathcal{I}| \times \mathcal{D}^1_t$ such messages, the message computation (including factor precomputation) at this factor node over $N_m$ max-sum iterations incurs a total of $\mathcal{O}(N_m |\mathcal{I}|^2 \mathcal{D}_t^{|\mathcal{I}|} + \mathcal{D}_t^{|\mathcal{I}|}  t^3  |\mathcal{I}|) = \mathcal{O}(N_m \mathcal{D}_t^{|\mathcal{I}|}  |\mathcal{I}|  (t^3 + |\mathcal{I}|))$ time. At the end of the message iteration phase, each variable node sends its latest update of $m_{\mathbf{x}^{(i)}\rightarrow \varphi_t^{\mathcal{I}}}(h)$ to an arbitrary factor $\mathcal{I}$ in its neighborhood. The receiving factor then uses~\eqref{eq:13} to generate the optimal value for $\mathbf{x}^{(i)}$, which incurs $\mathcal{O}(\mathcal{D}_t ^{|\mathcal{I}|})$ time per BO iteration. We assume the cost of sending and receiving messages between nodes are negligible and therefore omit them from the analysis for clarity. 
\subsubsection{Variable Node.} For each decentralized optimization iteration, computing each message at a variable node incurs $\mathbf{x}^{(i)}$ is $\mathcal{O}(|\mathcal{A}(i)|)$ time. Since there are $|\mathcal{A}(i)| \times \mathcal{D}^1_t$ such messages, the message computation at this variable node over $N_m$ iterations then incurs $\mathcal{O}(N_m |\mathcal{A}(i)|^2 \mathcal{D}^1_t)$ time. 
\section{Selecting Factor Graph Representation}
\label{fgr}
The formulation of our distributed message passing algorithm in Section~\ref{dBO} requires a specification of the input partition $\mathcal{U} \subseteq 2^{\mathcal{S}}$ that underlies our additive acquisition function~\eqref{eq:11}. This can be either specified manually by inspecting the data \cite{Kandasamy15} or learned from data~\cite{Garnett17,Zi17}. We adopt the recent approach of~\citeauthor{Garnett17}~\shortcite{Garnett17} by associating each factor graph candidate with an equivalent kernel of the resulting additive GP~\cite{Duvenaud11}. This allows~\eqref{eq:11} to be reformulated as a weighted average with respect to the posterior of $\mathcal{U}$ given the data $\mathfrak{D}_t\triangleq\{(\mathbf{x}_i,y_i) \}_{i=1,\ldots,t}$ of selected input queries and their noisy  outputs observed from evaluating $f$ after $t$ iterations:
\begin{equation}
\begin{array}{c} 
\displaystyle\sum_\mathcal{U} p(\mathcal{U} | \mathfrak{D}_t) \sum_{\mathcal{I}\in\mathcal{U}}\varphi^{\mathcal{I}}_t(\mathbf{x}^{\mathcal{I}})
\approx k^{-1}\sum_{i=1}^k \sum_{\mathcal{I}\in\mathcal{U}_i}\varphi^{\mathcal{I}}_t(\mathbf{x}^{\mathcal{I}}) 
\end{array}
\label{eq:13a}
\end{equation} 
where $\mathcal{U}_1,\ldots,\mathcal{U}_k$ are i.i.d samples drawn from $p(\mathcal{U} | \mathfrak{D}_t)$ via MCMC sampling \cite{Garnett17}. Interestingly, the RHS of~\eqref{eq:13a} can be equivalently interpreted as a sum of augmented local acquisition functions 
$k^{-1}\varphi_t^{\mathcal{I}}(\mathbf{x})$ induced from subsets $\mathcal{I}$ of input components in the union of input partitions $\mathcal{U}_1 \cup \ldots \cup \mathcal{U}_k$. As such, the augmented additive acquisition function~\eqref{eq:13a} 
can similarly be efficiently optimized in a decentralized manner using our distributed message passing algorithm in Section~\ref{dBO}.
It also allows the interdependent effects of different input components on the output of $f$ to be discovered and simultaneously exploited to find the global maximizer~\cite{Garnett17}.
%
\section{Proof of Theorem~\ref{theo:1}}
\label{app:a}
To prove Theorem~\ref{theo:1}, we first establish the following results:
\begin{lemma}
\label{lem:1}
Given $\delta \in (0, 1)$, let $\beta_{t} \triangleq 2\log(|\mathcal{D}||\mathcal{U}|\pi_{t}/\delta)$ with $\pi_t = \pi^2t^2/6$. Then, with probability of at least $1 - \delta$,
$$
\left|f(\mathbf{x}) - \sum_{\mathcal{I} \in \mathcal{U}} \mu^\mathcal{I}_{t-1}(\mathbf{x}^\mathcal{I})\right| \leq \beta_t^{1/2}\sum_{\mathcal{I} \in \mathcal{U}} \sigma_{t-1}^\mathcal{I}(\mathbf{x}^\mathcal{I}) 
$$
for all $\mathbf{x}\in\mathcal{D}$ and $t \in\mathbb{N}$ where $\mu_{t-1}^\mathcal{I}(\mathbf{x}^\mathcal{I})$ and $\sigma_{t-1}^\mathcal{I}(\mathbf{x}^\mathcal{I})$ are previously defined in~\eqref{eq:9}.
\end{lemma}
\begin{proof}
For all $\mathbf{x} \in \mathcal{D}, \mathcal{I} \in \mathcal{U}$, and $t \in\mathbb{N}$,
\begin{equation}
f_\mathcal{I}(\mathbf{x}^\mathcal{I}) \sim \mathcal{N}(\mu_{t-1}^\mathcal{I}(\mathbf{x}^\mathcal{I}), \sigma_{t-1}^\mathcal{I}(\mathbf{x}^\mathcal{I})^2) \ .
\label{eq:a2}
\end{equation}
Let $r \triangleq (f_\mathcal{I}(\mathbf{x}^\mathcal{I}) - \mu_{t-1}^\mathcal{I}(\mathbf{x}^\mathcal{I})) / \sigma_{t-1}^\mathcal{I}(\mathbf{x}^\mathcal{I})$. Then,~\eqref{eq:a2} implies $r \sim \mathcal{N}(0, 1)$ and hence, $\mathrm{Pr}(|r| \leq \beta_t^{1/2}) \geq 1 - \mathrm{exp}(-\beta_t/2)$. That is,
$$
\mathrm{Pr}\left(|f_\mathcal{I}(\mathbf{x}^\mathcal{I}) - \mu_{t-1}^\mathcal{I}(\mathbf{x}^\mathcal{I})| \leq \beta^{1/2}_t\sigma_{t-1}^\mathcal{I}(\mathbf{x}^\mathcal{I})\right) \geq 1 - \exp(-\beta_t/2)
$$ 
which, by applying the union bound over all tuples $(\mathbf{x} \in \mathcal{D}, \mathcal{I} \in \mathcal{U}, t \in\mathbb{N})$, implies
$$
\begin{array}{l}
\displaystyle\mathrm{Pr}\left(\forall \mathbf{x},\mathcal{I}, t \in\mathbb{N}\ \ 
|f_\mathcal{I}(\mathbf{x}^\mathcal{I}) - \mu_{t-1}^\mathcal{I}(\mathbf{x}^\mathcal{I})| \leq \beta^{1/2}_t\sigma_{t-1}^\mathcal{I}(\mathbf{x}^\mathcal{I})\right) \\
\displaystyle\geq 1 - |\mathcal{D}||\mathcal{U}|\sum_{t=1}^{\infty}\exp(-\beta_t/2) = 1 - \delta \ .
\end{array} 
$$
This means with probability of at least $1 - \delta$, the following inequalities hold simultaneously for all tuples $(\mathbf{x},\mathcal{I}, t)$:
\begin{equation}
\begin{array}{rcl}
f_\mathcal{I}(\mathbf{x}^\mathcal{I}) &\leq& \mu_{t-1}^\mathcal{I}(\mathbf{x}^\mathcal{I}) + \beta_t^{1/2}\sigma_{t-1}^\mathcal{I}(\mathbf{x}^\mathcal{I}) \ ,\\
f_\mathcal{I}(\mathbf{x}^\mathcal{I}) &\geq& \mu_{t-1}^\mathcal{I}(\mathbf{x}^\mathcal{I}) - \beta_t^{1/2}\sigma_{t-1}^\mathcal{I}(\mathbf{x}^\mathcal{I}) \ . 
\label{eq:a6}
\end{array}
\end{equation}
Summing over $\mathcal{I} \in \mathcal{U}$ on both sides of the above inequalities yields
$$
\hspace{-1.7mm}
\begin{array}{l}
\displaystyle f(\mathbf{x}) = \sum_{\mathcal{I} \in \mathcal{U}}f_\mathcal{I}(\mathbf{x}^\mathcal{I}) \leq \sum_{\mathcal{I} \in \mathcal{U}}\mu_{t-1}^\mathcal{I}(\mathbf{x}^\mathcal{I}) 
+ \beta_t^{1/2}\sum_{\mathcal{I}\in\mathcal{U}}\sigma_{t-1}^\mathcal{I}(\mathbf{x}^\mathcal{I}) \ ,\\
\displaystyle f(\mathbf{x})  =\sum_{\mathcal{I} \in \mathcal{U}}f_\mathcal{I}(\mathbf{x}^\mathcal{I}) \geq \sum_{\mathcal{I} \in \mathcal{U}}\mu_{t-1}^\mathcal{I}(\mathbf{x}^\mathcal{I}) - \beta_t^{1/2}\sum_{\mathcal{I}\in\mathcal{U}}\sigma_{t-1}^\mathcal{I}(\mathbf{x}^\mathcal{I}) \ .
\end{array}
$$
That is, for all pairs of $(\mathbf{x}, t)$,
\begin{equation}
\displaystyle\left|f(\mathbf{x}) - \sum_{\mathcal{I} \in \mathcal{U}} \mu^\mathcal{I}_{t-1}(\mathbf{x}^\mathcal{I})\right| \leq \beta_t^{1/2}\sum_{\mathcal{I} \in \mathcal{U}} \sigma_{t-1}^\mathcal{I}(\mathbf{x}^\mathcal{I}) \ .
\label{eq:a9}
\end{equation}
Since~\eqref{eq:a6} holds simultaneously for all tuples $(\mathcal{I},\mathbf{x}, t)$ with probability of at least $1 - \delta$,~\eqref{eq:a9} also holds simultaneously for all pairs of $(\mathbf{x}, t)$ with probability of at least $1 - \delta$. 
\end{proof}
\begin{lemma}
\label{lem:2}
For all $t \in\mathbb{N}$, if 
\begin{equation}
\displaystyle\left|f(\mathbf{x}) - \sum_{\mathcal{I}\in\mathcal{U}}\mu_{t-1}^\mathcal{I}(\mathbf{x}^\mathcal{I})\right| \leq \beta_t^{1/2}\sum_{\mathcal{I} \in \mathcal{U}}\sigma_{t-1}^\mathcal{I}(\mathbf{x}^\mathcal{I})  
\label{eq:a10}
\end{equation}
for all $\mathbf{x} \in \mathcal{D}$, then $r_t \leq 2\beta_t^{1/2}\sum_{\mathcal{I} \in\mathcal{U}}\sigma_{t-1}^\mathcal{I}(\mathbf{x}_t^\mathcal{I})$.
\end{lemma}
\begin{proof} 
By definition, 
$\mathbf{x}_{t} = \argmax_{\mathbf{x}\in\mathcal{D}} \sum_{\mathcal{I}}\varphi_t^\mathcal{I}(\mathbf{x}^\mathcal{I})$
and hence, 
$\sum_{\mathcal{I}}\varphi_t^\mathcal{I}(\mathbf{x}^\mathcal{I}_t)\geq\sum_{\mathcal{I}}\varphi_t^\mathcal{I}(\mathbf{x}^\mathcal{I}_*)$. 
This implies
$$
\begin{array}{l}
\displaystyle\sum_{\mathcal{I}\in\mathcal{U}} \mu_{t-1}^\mathcal{I}(\mathbf{x}_t^\mathcal{I}) + \beta_t^{1/2}\sigma_{t-1}^\mathcal{I}(\mathbf{x}_t^\mathcal{I}) \\ 
\geq\displaystyle\sum_{\mathcal{I}\in\mathcal{U}} \mu_{t-1}^\mathcal{I}(\mathbf{x}_\ast^\mathcal{I}) + \beta_t^{1/2}\sigma_{t-1}^\mathcal{I}(\mathbf{x}_\ast^\mathcal{I})\\
\geq f(\mathbf{x}_\ast)  
\end{array}
$$
where the second inequality follows directly from~\eqref{eq:a10}. 
Then, 
\begin{equation}
\begin{array}{rcl}
r_t &=& f(\mathbf{x}_\ast) - f(\mathbf{x}_t) \\
&\leq&\displaystyle \sum_{\mathcal{I}\in\mathcal{U}}\mu_{t-1}^\mathcal{I}(\mathbf{x}_t^\mathcal{I}) + \beta_t^{1/2} \sum_{\mathcal{I}\in\mathcal{U}}\sigma_{t-1}^\mathcal{I}(\mathbf{x}_t^\mathcal{I}) - f(\mathbf{x}_t)\ .
\end{array}
\label{eq:a12}
\end{equation}
On the other hand, also by~\eqref{eq:a10},
\begin{equation}
\sum_{\mathcal{I}\in\mathcal{U}}\mu_{t-1}^\mathcal{I}(\mathbf{x}_t^\mathcal{I}) - f(\mathbf{x}_t) \leq \beta_t^{1/2}\sum_{\mathcal{I}\in\mathcal{U}}\sigma_{t-1}^\mathcal{I}(\mathbf{x}_t^\mathcal{I}) \ .
\label{eq:a13}
\end{equation}
Plugging~\eqref{eq:a13} into~\eqref{eq:a12} yields
$$
r_t \leq 2\beta_t^{1/2}\sum_{\mathcal{I}\in\mathcal{U}}\sigma_{t-1}^\mathcal{I}(\mathbf{x}_t^\mathcal{I}) \ .
$$
\end{proof}
%
\noindent
\emph{Main Proof}. Lemma~\ref{lem:1} guarantees that~\eqref{eq:a10} of Lemma~\ref{lem:2} holds universally for all pairs of $(\mathbf{x}, t)$ with probability of at least $1 - \delta$. As such, Theorem~\ref{theo:1} follows.
\section{Proof of Theorem~\ref{theo:2}}
\label{app:b}
From Theorem~\ref{theo:1},
\begin{equation}
\begin{array}{rcl}
\displaystyle\sum_{t=1}^{n} r^2_t &\leq& \displaystyle\sum_{t=1}^{n} 4\beta_{t} \left( \sum_{\mathcal{I} \in \mathcal{U}} \sigma^{\mathcal{I}}_{t-1}(\mathbf{x}^{\mathcal{I}}_t) \right)^2  \\
&\leq&\displaystyle \sum_{t=1}^{n} \sum_{\mathcal{I} \in \mathcal{U}} 4\beta_t |\mathcal{U}| \sigma^{\mathcal{I}}_{t-1}(\mathbf{x}^{\mathcal{I}}_t)^2  
\end{array}
\label{eq:b1}
\end{equation}
where the second inequality is due to the Cauchy-Schwarz inequality.
To prove Theorem~\ref{theo:2}, we first introduce a mild assumption on the relationship between the posterior variance $\sigma^{\mathcal{I}}_{t-1}(\mathbf{x}^{\mathcal{I}}_t)^2$ conditioned on noisy outputs of $f$ and its counterpart $\widehat{\sigma}^{\mathcal{I}}_{t-1}(\mathbf{x}^{\mathcal{I}}_t)^2$ conditioned on noisy outputs of $f_{\mathcal{I}}$ (i.e., assuming hypothetically that they are available) which are  
perturbed by the same i.i.d. Gaussian noise $\epsilon \sim \mathcal{N}(0, \sigma_n^2)$, i.e., 
$$
\widehat{\sigma}^\mathcal{I}_{t-1}(\mathbf{x}_t^{\mathcal{I}})^2 \triangleq \sigma_0(\mathbf{x}_t^{\mathcal{I}},\mathbf{x}_t^{\mathcal{I}}) - 
\mathbf{k}^{\mathcal{I}^{\top}}_{\mathbf{x}_t}
(\mathbf{K}^{\mathcal{I}}
+ \sigma^2_n\mathbf{I})^{-1}
\mathbf{k}_{\mathbf{x}_t}^{\mathcal{I}}
$$
where $\mathbf{k}_{\mathbf{x}_t}^{\mathcal{I}} \triangleq (\sigma_0^\mathcal{I}(\mathbf{x}^{\mathcal{I}}_t, \mathbf{x}_i^{\mathcal{I}}))^\top_{i=1,\ldots,t-1}$ and $\mathbf{K}^{\mathcal{I}} \triangleq (\sigma_0(\mathbf{x}^{\mathcal{I}}_i,\mathbf{x}^{\mathcal{I}}_j))_{i,j=1,\ldots,t-1}$. 
\begin{assumption}
For any sequence of input queries $\langle\mathbf{x}_t^{\mathcal{I}}\rangle^{n}_{t=1}$, there exists an arbitrary decreasing function $h:\mathbb{N}\rightarrow\mathbb{R}$ such that $\lim_{t \rightarrow \infty} h(t) > 0$ and 
$$
	h(t-1)\ \sigma_{t-1}^{\mathcal{I}}(\mathbf{x}_t^{\mathcal{I}})^2 \leq \widehat{\sigma}_{t-1}^{\mathcal{I}}(\mathbf{x}_t^{\mathcal{I}})^2
$$
for $t\in\mathbb{N}$.
\end{assumption}
Assumption 4 allows us to bound the posterior variance $\sigma^{\mathcal{I}}_{t-1}(\mathbf{x}^{\mathcal{I}}_t)^2$ from the above by
\begin{equation}
	\sigma^{\mathcal{I}}_{t-1}(\mathbf{x}^{\mathcal{I}}_t)^2 \leq\frac{1}{h(t-1)}\widehat{\sigma}^{\mathcal{I}}_{t-1}(\mathbf{x}^{\mathcal{I}}_t)^2\ .
	\label{eq:b3}
\end{equation}
Plugging~\eqref{eq:b3} into~\eqref{eq:b1} yields
\begin{equation}
\begin{array}{rcl}
	\displaystyle\sum_{t=1}^{n} r^2_t &\leq&\displaystyle\sum_{t=1}^n\sum_{\mathcal{I}\in\mathcal{U}}\frac{4\beta_t|\mathcal{U}|}{h(t-1)}\widehat{\sigma}^{\mathcal{I}}_{t-1}(\mathbf{x}^{\mathcal{I}}_t)^2  \\
	&\leq& \displaystyle F\sum_{t=1}^n\sum_{\mathcal{I}\in\mathcal{U}}\widehat{\sigma}^{\mathcal{I}}_{t-1}(\mathbf{x}^{\mathcal{I}}_t)^2
\end{array}
	\label{eq:b4}
\end{equation}
where $F \triangleq 4\beta_n|\mathcal{U}|/\lim_{t \rightarrow \infty}h(t)$
and the last inequality is due to the monotonic increase of $\beta_t$ and $1/	h(t)$ in $t$. 
Finally, to relate the total posterior variance $\sum_{t=1}^n \widehat{\sigma}^{\mathcal{I}}_{t-1}(\mathbf{x}^{\mathcal{I}}_t)^2$ to the maximum information gain $\gamma^{\mathcal{I}}_n$ (Definition~\ref{def:1}) for each factor function $f_\mathcal{I}$, we exploit the monotonically increasing property of the following function $g(s) = s / \log(1 + s)$ with $s = \sigma_n^{-2}\widehat{\sigma}_{t-1}^\mathcal{I}(\mathbf{x}_t^\mathcal{I})^2$, as detailed below.

Specifically, since  the function $g(s) = s/\log(1 + s)$ increases monotonically on $[0,\infty)$ and  $\sigma_n^{-2}\widehat{\sigma}_{t-1}^\mathcal{I}(\mathbf{x}_t^{\mathcal{I}})^2 \leq \sigma_n^{-2}\sigma_0^{\mathcal{I}}(\mathbf{x}_t^{\mathcal{I}}, \mathbf{x}_t^{\mathcal{I}}) \leq \sigma_n^{-2}\sigma_s^2$ where\footnote{The first inequality follows because $\widehat{\sigma}_0^{\mathcal{I}}(\mathbf{x}_t^{\mathcal{I}}, \mathbf{x}_t^{\mathcal{I}})= \sigma_0^{\mathcal{I}}(\mathbf{x}_t^{\mathcal{I}}, \mathbf{x}_t^{\mathcal{I}})$ and the GP posterior variance is always non-increasing, 
i.e., $\widehat{\sigma}_0^{\mathcal{I}}(\mathbf{x}_t^{\mathcal{I}}, \mathbf{x}_t^{\mathcal{I}}) \geq \widehat{\sigma}_{t-1}^{\mathcal{I}}(\mathbf{x}_t^{\mathcal{I}}, \mathbf{x}_t^{\mathcal{I}})$ for all $t \in \mathbb{N}$.} $\sigma^2_s$ denotes the signal variance, it follows that
$$
	\widehat{\sigma}_{t-1}^{\mathcal{I}}(\mathbf{x}_t^{\mathcal{I}})^2 \leq \sigma_n^2\ g(\sigma_n^{-2}\sigma_s^2)\log\left(1 + \sigma_n^{-2}\widehat{\sigma}_{t-1}^\mathcal{I}(\mathbf{x}_t^\mathcal{I})^2\right)\ .
$$
Applying this result to~\eqref{eq:b4} gives
\begin{equation}
\begin{array}{rcl}
	\displaystyle\sum_{t=1}^{n} r^2_t &\leq& \displaystyle\beta_nC\sum_{\mathcal{I}\in\mathcal{U}}\frac{1}{2}\sum_{t=1}^n\log\left(1 + \sigma_n^{-2}\widehat{\sigma}_{t-1}^\mathcal{I}(\mathbf{x}_t^\mathcal{I})^2\right) \\
	&\leq& \displaystyle\beta_nC\sum_{\mathcal{I}\in\mathcal{U}} \gamma_n^\mathcal{I} \ \ =\ \ C\beta_n\gamma_n   
\end{array}
	\label{eq:b6}
\end{equation}
where $C \triangleq 2\sigma_n^2F g(\sigma_n^{-2}\sigma_s^2)$ 
and the second inequality follows directly from
Lemma $5.3$ of \citeauthor{Srinivas10}~\shortcite{Srinivas10}. Finally, applying Cauchy-Schwarz to the LHS of~\eqref{eq:b6} yields
$$
	R_n \triangleq \sum_{t=1}^n r_t \leq \sqrt{Cn\beta_n\gamma_n} \leq \sqrt{C|\mathcal{U}|}\sqrt{n\beta_n\gamma_n^{\mathcal{I}_\ast}} 
$$
where $\gamma_n^{\mathcal{I}_\ast} \triangleq \max_{\mathcal{I}}\gamma_n^\mathcal{I}$. Since Theorem $5$ of~\citeauthor{Srinivas10}~\shortcite{Srinivas10} has shown a sublinear growth of the maximum information gain $\gamma_n^{\mathcal{I}_\ast}$ in $n$ for the factor function $f_{\mathcal{I}_\ast}$,
it follows that $\lim_{n\rightarrow \infty} \sqrt{n\beta_n\gamma_n^{\mathcal{I}_\ast}}/n = 0$.
 This consequently implies $\lim_{n\rightarrow\infty} R_n/n = 0$ or, equivalently, $R_n \leq (Cn\beta_n\gamma_n)^{1/2} \in o(n)$ when 
 Theorem~\ref{theo:1} holds.
 However, since Theorem~\ref{theo:1} only holds with probability of at least $1 - \delta$,
 this implies $R_n \leq (Cn\beta_n\gamma_n)^{1/2} \in o(n)$ with probability of at least $1 - \delta$.
\section{Proof of Theorem~\ref{theo:3}}
\label{app:c}
As previously discussed in Section~\ref{analysis}, Theorem~\ref{theo:2} is unfortunately rendered void with infinite, continuous input spaces which cause 
$\langle\beta_t\rangle_t$ to tend to infinity. 
Fortunately, much of the proof of Lemma~\ref{lem:1}
can still be reused to bound $|f(\mathbf{x}) - \sum_{\mathcal{I}\in\mathcal{U}}\mu_{t-1}^\mathcal{I}(\mathbf{x}^\mathcal{I})|$ for an arbitrary sequence of input queries selected by our DEC-HBO algorithm (Lemma~\ref{lem:3}) and an arbitrary finite discretization of the input space (Lemma~\ref{lem:4}).
The real challenge is, however, to extend such an analysis to an arbitrary input not among the discretized inputs.
 As shall be elaborated later, this can be achieved by exploiting the Lipschitz continuity of the objective function $f$ (see Assumption~\ref{aaa}) to essentially bound the difference between outputs of $f$ at two separate inputs in terms of their proximity, thus effectively ``projecting'' our bound beyond the discretization, which constitutes the central theme of this proof.
%
%
\begin{lemma}
\label{lem:3}
Given $\delta \in (0, 1)$, let $\beta_t \triangleq 2\log(|\mathcal{U}|\pi_t/\delta)$, and $\langle\mathbf{x}_t\rangle_{t=1}^\infty$ denote an arbitrary sequence of input queries selected by our DEC-HBO algorithm. Then, with probability of at least $1 - \delta$,
\begin{equation}
\left|f(\mathbf{x}_t) - \sum_{\mathcal{I}\in\mathcal{U}}\mu_{t-1}^\mathcal{I}(\mathbf{x}_t^\mathcal{I})\right| \leq \beta_t^{1/2}\sum_{\mathcal{I}\in\mathcal{U}}\sigma_{t-1}^\mathcal{I}(\mathbf{x}_t^\mathcal{I}) \label{eq:c1}
\end{equation}
for all $t \in \mathbb{N}$.
\end{lemma}
\begin{proof} 
This result is similar to that of Lemma~\ref{lem:1} but restricted to the (infinitely) countable sequence of  input queries selected by our DEC-HBO algorithm. Most arguments established in the proof of Lemma~\ref{lem:1} can be reused here, except that the union bound does not have to be applied to the entire input space, hence not causing $\langle\beta_t\rangle_t$ to blow up to infinity. 
In particular, using a similar argument to that of Lemma~\ref{lem:1}, for a given tuple $(\mathcal{I}, t)$, 
$$
\displaystyle\mathrm{Pr}\left(\left|f_\mathcal{I}(\mathbf{x}^\mathcal{I}_t) - \mu_{t-1}^\mathcal{I}(\mathbf{x}_t^\mathcal{I})\right| \leq \beta_t^{1/2}\sigma_{t-1}^\mathcal{I}(\mathbf{x}_t^\mathcal{I})\right)\geq 1 - \exp(-{\beta_t}/{2}) . 
$$
Applying the union bound over all tuples $(\mathcal{I}, t)$ yields
\begin{equation}
\begin{array}{l}
\displaystyle\mathrm{Pr}\left(\forall\mathcal{I}, t\in\mathbb{N}\ \ |f_\mathcal{I}(\mathbf{x}^\mathcal{I}_t) - \mu_{t-1}^\mathcal{I}(\mathbf{x}_t^\mathcal{I})| \leq \beta_t^{1/2}\sigma_{t-1}^\mathcal{I}(\mathbf{x}_t^\mathcal{I})\right) \\
\displaystyle\geq 1 - |\mathcal{U}|\sum_{t=1}^{\infty}\exp(-{\beta_t}/{2}) = 1 - \delta\ . 
\end{array}
\label{eq:c3}
\end{equation}
Using a similar argument as that of Lemma~\ref{lem:1}, with probability of at least $1 - \delta$, the following inequality holds simultaneously for all $t \in \mathbb{N}$:
$$
\left|f(\mathbf{x}_t) - \sum_{\mathcal{I}\in\mathcal{U}}\mu_{t-1}^\mathcal{I}(\mathbf{x}_t^\mathcal{I})\right| \leq \beta_t^{1/2}\sum_{\mathcal{I}\in\mathcal{U}}\sigma_{t-1}^\mathcal{I}(\mathbf{x}_t^\mathcal{I}) \ . 
$$
\end{proof}
\begin{lemma}
\label{lem:4}
Given $\delta \in (0, 1)$, let $\mathcal{D}_t$ denote an arbitrary finite discretisation of the input space, and $\beta_{t} \triangleq 2\log(|\mathcal{D}_t||\mathcal{U}|\pi_{t}/\delta)$.
Then, with probability of at least $1 - \delta$,
$$
\left|f(\mathbf{x}) - \sum_{\mathcal{I} \in \mathcal{U}} \mu^\mathcal{I}_{t-1}(\mathbf{x}^\mathcal{I})\right| \leq \beta_t^{1/2}\sum_{\mathcal{I} \in \mathcal{U}} \sigma_{t-1}^\mathcal{I}(\mathbf{x}^\mathcal{I}) 
$$
for all $\mathbf{x} \in \mathcal{D}_t$ and $t\in\mathbb{N}$.
\end{lemma}
\begin{proof} 
This result follows directly by applying Lemma~\ref{lem:1} to the finite discretization of the input space $\mathcal{D}_t$.
\end{proof}
Putting together the above results of Lemmas~\ref{lem:3} and~\ref{lem:4}, it is straightforward to see that the instantaneous regret $f(\mathbf{x}_\ast) - f(\mathbf{x}_t)$ can be bounded with high probability if $\mathbf{x}_\ast \in \mathcal{D}_t$. 
A tricky situation, however, arises when the global maximizer $\mathbf{x}_\ast$ is not among the discretized inputs. To resolve this, one possible approach is to relate the output of $\mathbf{x}_\ast$ to that of its closest discretized input.
If this can be achieved, then we can exploit Lemma~\ref{lem:4} to deliver a high-confidence bound on the instantaneous regret. 

Specifically, suppose that we choose a finite discretization $\mathcal{D}_t$ of the original input space $\mathcal{D} = [0, r]^d$  in iteration $t$ such that each dimension has $\tau_t$ uniformly-spaced discretized inputs; the exact value for $\tau_t$ will be determined later in Lemma~\ref{lem:5}.
That is, $|\mathcal{D}_t| = \tau_t^d$ and $\|\mathbf{x} - [\mathbf{x}]_t\|_1  \leq rd/\tau_t$ for all $\mathbf{x} \in \mathcal{D}_t$ where $[\mathbf{x}]_t$ denotes the closest discretized input in $\mathcal{D}_t$ to $\mathbf{x}$. Under this setting, we are now ready to bound the true output of the global maximizer in terms of the predicted output (i.e., using the sum of GP posterior means of the outputs of all factor functions) of its closest discretized input with high confidence:
\begin{lemma}
\label{lem:5}
Given $\delta \in (0, 1)$, let $\beta_t \triangleq 2\log(2|\mathcal{U}|\pi_t/\delta) + 2d\log(rdbt^2\sqrt{\log(2|\mathcal{U}|a/\delta)})$
where $a$ and $b$ are the Lipschitz constants defined previously in Assumption $3$. Then, 
with probability of at least $1 - \delta$,
$$
\displaystyle
\left|f(\mathbf{x}_\ast) - \sum_{\mathcal{I}\in\mathcal{U}}\mu_{t-1}^\mathcal{I}([\mathbf{x}^\mathcal{I}_\ast]_t)\right| \leq \beta_t^{1/2}\sum_{\mathcal{I}\in\mathcal{U}}\sigma_{t-1}^\mathcal{I}([\mathbf{x}^\mathcal{I}_\ast]_t) + \frac{1}{t^2} 
$$
for all $t \in\mathbb{N}$.
\end{lemma}
\begin{proof} 
Setting $L \triangleq b\sqrt{\log(2|\mathcal{U}|a/\delta)}$ in Assumption $3$ 
immediately implies that
$$
\left|f(\mathbf{x}) - f(\mathbf{x}')\right| \leq b\sqrt{\log\left(\frac{2|\mathcal{U}|a}{\delta}\right)}\|\mathbf{x}-\mathbf{x}'\|_1 
$$
holds simultaneously for all $\mathbf{x}, \mathbf{x}'\in\mathcal{D}$ with probability of at least $1 - \delta/2$. 
Then, 
$$
\begin{array}{rcl}
\left|f(\mathbf{x}_\ast) - f([\mathbf{x}_\ast]_t)\right| &\leq& \displaystyle b\sqrt{\log\left(\frac{2|\mathcal{U}|a}{\delta}\right)}\|\mathbf{x}_\ast-[\mathbf{x}_\ast]_t\|_1\\
&\leq& \displaystyle\frac{rdb}{\tau_t}\sqrt{\log\left(\frac{2|\mathcal{U}|a}{\delta}\right)} 
\end{array}
$$
holds simultaneously for all $t \in\mathbb{N}$ with probability of at least $1 - \delta/2$ such that the second inequality is due to the construction of $\mathcal{D}_t$ described earlier. Now, by choosing $\tau_t \triangleq rdbt^2\sqrt{\log(2|\mathcal{U}|a/\delta)}$, 
\begin{equation}
\left|f(\mathbf{x}_\ast) - f([\mathbf{x}_\ast]_t)\right| \leq \frac{1}{t^2}  
\label{eq:c10}
\end{equation}
holds simultaneously for all $t \in\mathbb{N}$ with probability of at least $1 - \delta/2$. This also implies $|\mathcal{D}_t| = (rdbt^2\sqrt{\log(2|\mathcal{U}|a/\delta)})^d$. 
By applying this choice of $\mathcal{D}_t$ and $\delta/2$ to Lemma~\ref{lem:4},
\begin{equation}
\left|f([\mathbf{x}_\ast]_t) - \sum_{\mathcal{I}\in\mathcal{U}}\mu_{t-1}^\mathcal{I}([\mathbf{x}^\mathcal{I}_\ast]_t)\right| \leq\beta_t^{1/2}\sum_{\mathcal{I}\in\mathcal{U}}\sigma_{t-1}^\mathcal{I}([\mathbf{x}^\mathcal{I}_\ast]_t) 
\label{eq:c11}
\end{equation}
holds simultaneously for all $t \in\mathbb{N}$ with probability of at least $1 - \delta/2$ where $\beta_t = 2\log(2|\mathcal{D}_t||\mathcal{U}|\pi_t/\delta)$. By combining~\eqref{eq:c10} and~\eqref{eq:c11} and using the union bound, the following inequality holds simultaneously for all $t \in\mathbb{N}$ with probability of at least $1 - \delta$:
$$
\left|f([\mathbf{x}_\ast]_t) - \sum_{\mathcal{I}\in\mathcal{U}}\mu_{t-1}^\mathcal{I}([\mathbf{x}^\mathcal{I}_\ast]_t)\right| \leq  \beta_t^{1/2}\sum_{\mathcal{I}\in\mathcal{U}}\sigma_{t-1}^\mathcal{I}([\mathbf{x}_\ast]_t^\mathcal{I}) + \frac{1}{t^2} 
$$
for $\beta_t = 2\log(2|\mathcal{D}_t||\mathcal{U}|\pi_t/\delta) = 2\log(2|\mathcal{U}|\pi_t/\delta) + 2d\log(rdbt^2\sqrt{\log(2|\mathcal{U}|a/\delta)})$. The second equality follows directly from the above choice of $\mathcal{D}_t$.
\end{proof}
Using Lemma~\ref{lem:5}, we are now ready to bound the instantaneous regret $f(\mathbf{x}_\ast) - f(\mathbf{x}_t)$:
\begin{lemma}
Given $\delta \in (0, 1)$, with probability at least $1 - \delta$,
$$
\displaystyle r_t \leq 2\beta_t^{1/2}\sum_{\mathcal{I}\in\mathcal{U}}\sigma_{t-1}^\mathcal{I}(\mathbf{x}_t^\mathcal{I}) + \frac{1}{t^2}  
$$
for all $t \in\mathbb{N}$ where $\beta_t$ is previously defined in Lemma~\ref{lem:5}.
\end{lemma}
\begin{proof} 
By applying Lemma~\ref{lem:5} for $\delta/2$, 
the following inequality 
\begin{equation}
\begin{array}{rcl}
f(\mathbf{x}_\ast) &\leq& \displaystyle\frac{1}{t^2} +  \sum_{\mathcal{I}\in\mathcal{U}}\mu_{t-1}^\mathcal{I}([\mathbf{x}^\mathcal{I}_\ast]_t) + \beta_t^{1/2}\sigma_{t-1}^\mathcal{I}([\mathbf{x}^\mathcal{I}_\ast]_t)  \\
&\leq& \displaystyle\frac{1}{t^2} + \sum_{\mathcal{I}\in\mathcal{U}}\mu_{t-1}^\mathcal{I}(\mathbf{x}_t^\mathcal{I}) + \beta_t^{1/2}\sigma_{t-1}^\mathcal{I}(\mathbf{x}_t^\mathcal{I})  
\end{array}
\label{eq:c14}
\end{equation}
holds simultaneously for all $t \in\mathbb{N}$ with probability of at least $1 - \delta/2$ where $\beta_t = 2\log(4|\mathcal{U}|\pi_t/\delta) + 2d\log(dt^2br\sqrt{\log(4|\mathcal{U}|a/\delta)})$. Note that the second inequality in~\eqref{eq:c14} always hold with certainty due to the definition of $\mathbf{x}_t$. So,
\begin{equation}
\begin{array}{l}
r_t \\
\displaystyle = f(\mathbf{x}_\ast) - f(\mathbf{x}_t) \\
\leq\displaystyle \sum_{\mathcal{I}\in\mathcal{U}}\left(\mu_{t-1}^\mathcal{I}(\mathbf{x}_t^\mathcal{I}) + \beta_t^{1/2}\sigma_{t-1}^\mathcal{I}(\mathbf{x}_t^\mathcal{I})\right) + \frac{1}{t^2} - f(\mathbf{x}_t) \\
= \displaystyle \left(\sum_{\mathcal{I}\in\mathcal{U}}\mu_{t-1}^\mathcal{I}(\mathbf{x}_t^\mathcal{I}) - f(\mathbf{x}_t)\right) + \beta_t^{1/2}\sum_{\mathcal{I}\in\mathcal{U}}\sigma_{t-1}^\mathcal{I}(\mathbf{x}_t^\mathcal{I}) +\frac{1}{t^2} \\
\leq \displaystyle 2\beta_t^{1/2}\sum_{\mathcal{I}\in\mathcal{U}}\sigma_{t-1}^\mathcal{I}\left(\mathbf{x}_t^\mathcal{I}\right) + \frac{1}{t^2} 
\end{array}
\label{eq:c15}
\end{equation}
where the first inequality holds with probability of at least $1 - \delta/2$ due to~\eqref{eq:c14} while the second inequality holds with the same probability of at least $1 - \delta/2$ for $\beta_t \geq 2\log(2|\mathcal{U}|\pi_t/\delta)$ by Lemma~\ref{lem:3}.\footnote{In Lemma~\ref{lem:3}, $\beta_t \geq 2\log(2|\mathcal{U}|\pi_t/\delta)$ is the minimum threshold for which~\eqref{eq:c1} holds.} As such, with $\beta_t = 2\log(4|\mathcal{U}|\pi_t/\delta) + 2d\log(dt^2br\sqrt{\log(4|\mathcal{U}|a/\delta)})$, we can conclude that both inequalities hold simultaneously for all $t\in\mathbb{N}$ with probability of at least $1 - \delta/2$ each. Applying union bound over them guarantees that~\eqref{eq:c15} holds with probability of at least $1 - \delta$. \end{proof}
\noindent\emph{Main Proof}. 
By replicating the arguments used in the proof of Theorem~\ref{theo:2} in Appendix~\ref{app:b} (specifically,~\eqref{eq:b1} to~\eqref{eq:b6})\footnote{Note that these arguments hold regardless of the choice of $\beta_t$. So, they can be reused in Theorem~\ref{theo:3} even though our choice of $\beta_t$ has changed.}, it can be derived that
\begin{equation}
4\sum_{t=1}^n \beta_t \left(\sum_{\mathcal{I}\in\mathcal{U}} \sigma_{t-1}^\mathcal{I}(\mathbf{x}_t^\mathcal{I})\right)^2 \leq C\beta_n\gamma_n 
\label{eq:c16}
\end{equation}
where $C$ and $\gamma_n$ are previously defined in Appendix~\ref{app:b} and Definition~\ref{def:1}, respectively. By applying the Cauchy-Schwarz inequality on the LHS of~\eqref{eq:c16}, 
$$
\left(\sum_{t=1}^n 2\beta_t^{1/2}\sum_{\mathcal{I}\in\mathcal{U}} \sigma_{t-1}^\mathcal{I}(\mathbf{x}_t^\mathcal{I})\right)^2 \leq Cn\beta_n\gamma_n
$$
which immediately implies
$$
\begin{array}{l}
\displaystyle\sum_{t=1}^n r_t 
= \displaystyle \sum_{t=1}^n 2\beta_t^{1/2}\sum_{\mathcal{I}\in\mathcal{U}} \sigma_{t-1}^\mathcal{I}(\mathbf{x}_t^\mathcal{I}) + \sum_{t=1}^n \frac{1}{t^2} \\
\leq \displaystyle\sqrt{Cn\beta_n\gamma_n} + \sum_{t=1}^n \frac{1}{t^2} \\
\displaystyle\leq \sqrt{Cn\beta_n\gamma_n} + \frac{\pi^2}{6} 
\end{array}
$$
where the last inequality follows from $\sum_{t=1}^n(1/t^2) \leq \sum_{t=1}^{\infty}(1/t^2) = \pi^2/6$. On the other hand, we have already established in Appendix~\ref{app:b} that $\lim_{n\rightarrow\infty}\sqrt{Cn\beta_n\gamma_n}/n = 0$, which means $\sqrt{Cn\beta_n\gamma_n} \in o(n)$. Since $\pi^2/6$ is a constant, $R_n = \sum_{t=1}^n r_t \leq \sqrt{Cn\beta_n\gamma_n} + \pi^2/6 \in o(n)$. 
%
\section{Synthetic Functions}
\label{app:d}
The \emph{Shekel} function is a $4$-dimensional function over the hypercube $[0, 10]^4$ given by $f(\mathbf{x}) \triangleq -\sum_{i=1}^{10}(\beta_i + \sum_{j=1}^4(x_j - C_{ij})^2)^{-1}$ where $\mathbf{x}\triangleq (x_j)_{j=1,\ldots,4}$, and 
$(\beta_i)_{i=1,\ldots,10}$ and $(C_{ij})_{i=1,\ldots,10,j=1,\ldots,4}$ are given as constants. 
It has one global minimum $f(\mathbf{x^\ast}) = -10.5364$. 
The \emph{Hartmann} function is a $6$-dimensional function over the hypercube $[0, 1]^6$ given by $f(\mathbf{x}) \triangleq - \sum_{i=1}^{4}\alpha_i\exp(-\sum_{j=1}^{6}A_{ij}(x_j - P_{ij})^2)$ where $\mathbf{x}\triangleq (x_j)_{j=1,\ldots,6}$, and
$(\alpha_i)_{i=1,\ldots,4}$, $(A_{ij})_{i=1,\ldots,4,j=1,\ldots,6}$, and $(P_{ij})_{i=1,\ldots,4,j=1,\ldots,6}$
are given as constants. It has one global minimum $f(\mathbf{x^\ast}) = -3.32237$. 
Lastly, the \emph{Michalewicz} function is a $10$-dimensional function over the hypercube $[0, \pi]^{10}$ given by $f(\mathbf{x}) \triangleq -\sum_{i=1}^{10}\sin(x_i)\ \sin^{2m}(ix^2_i/\pi)$
where $\mathbf{x}\triangleq (x_i)_{i=1,\ldots,10}$. It has one global minimum $f(\mathbf{x}^\ast) = -9.66015$.
\fi

\end{document}